\def\year{2019}\relax
\documentclass[letterpaper]{article} 
\usepackage{aaai19}  
\usepackage{times}  
\usepackage{helvet}  
\usepackage{courier}  
\usepackage{url}  
\usepackage{graphicx}  
\usepackage[utf8]{inputenc}
\usepackage{xspace}
\usepackage[textsize=tiny]{todonotes}
\definecolor{bondiblue}{rgb}{0.0, 0.58, 0.71}

\usepackage{bbm}
\usepackage{amsmath}
\usepackage{cases}
\usepackage{amssymb}
\usepackage{amsthm}
\usepackage{fullpage}
\usepackage{subfig}
\usepackage[titletoc]{appendix}
\usepackage[belowskip=-13pt,aboveskip=5pt]{caption}
\usepackage{tikz}
\usetikzlibrary{shapes,arrows,shadows,fit, decorations.pathmorphing}
\tikzstyle{state} = [draw, fill=black!15, text centered, minimum height=2em, rounded corners]
\tikzstyle{hstate}=[draw, text centered, minimum height=2em, rounded corners] 
\tikzstyle{txt} = [text centered, minimum height=2em]
\tikzstyle{txt_color} = [text centered, minimum height=2em, text=teal]

\tikzstyle{dgraph}=[->, line width=1.5pt]

\tikzstyle{ocont}=[circle,draw=blue!50,thick,fill=blue!10,minimum size=6mm,>=stealth] 
\tikzstyle{hcont}=[circle,draw=black!50,thick,fill=black!10,minimum size=6mm,>=stealth] 
\tikzstyle{contdec}=[circle,draw=red!50,thick,fill=red!10]  
\tikzstyle{cont}=[circle,draw=blue!50,thick,minimum size=6mm,line width=2pt,>=stealth]  
\tikzstyle{blackcont}=[circle,draw=black!50,thick,minimum size=6mm,line width=2pt,>=stealth]  
\tikzstyle{oval}=[ellipse,draw=blue!50,thick,minimum size=6mm,line width=2pt,>=stealth]  
\tikzstyle{disc}=[rectangle,draw=blue!50,thick,line width=1pt,minimum size=6mm]  

\tikzstyle{fillred}=[fill=red!20,thick]  
\tikzstyle{fillgreen}=[fill=green!20,thick]  
\tikzstyle{purered}=[fill=red]  
\tikzstyle{state}=[rectangle,fill=red!30]  
\tikzstyle{hstate}=[rectangle,fill=blue!20]  
\tikzstyle{sobs}=[fill=green!15,thick]  
\tikzstyle{fact}=[fill,minimum size=1.5mm,line width=2pt,>=stealth]
\tikzstyle{varfact}=[draw,minimum size=1.5mm,line width=2pt,>=stealth]
\tikzstyle{sep}=[rectangle,draw=magenta!50,thick,minimum size=6mm]  
\tikzstyle{det}=[fill=red!15,rectangle,draw=red!50,thick,minimum size=6mm]  
\tikzstyle{dethid}=[diamond,draw=red!50,thick,minimum size=6mm]  
\tikzstyle{lineball}=[fill,-*,draw=red!50,line width=1.5pt]
\tikzstyle{redball}=[mark=*,mark options={fill=red!50,draw=red},mark size=0.5pt]
\tikzstyle{greenball}=[mark=*,mark options={fill=green!50,draw=green},mark size=0.5pt]
\tikzstyle{dec}=[rectangle,draw=red!50,thick,minimum size=6mm]  
\tikzstyle{utility}=[diamond,draw=red!50,thick,minimum size=6mm]  
\tikzstyle{decutility}=[diamond,draw=red!50,thick,minimum size=6mm]  

\newcommand{\ie}{\emph{i.e.}}
\newcommand{\eg}{\emph{e.g.}}
\newcommand{\io}{\emph{i.o.}}

\newcommand{\M} {\mathcal{M}}
\newcommand{\s} {\mathcal{S}}

\newcommand{\Reals}{\mathbb R}

\newcommand{\bbE}{\mathbb E}
\newcommand{\bbP}{\mathbb P}

\newtheorem{theorem}{Theorem}

\newtheorem{remark}[theorem]{Remark}

\newcommand{\figref}[1]{Fig.~\ref{#1}}
\renewcommand{\eqref}[1]{Eq.~(\ref{#1})}

\usepackage{natbib}

\begin{document}

\title{Degenerate Feedback Loops in Recommender Systems}

\author{Ray Jiang, Silvia Chiappa, Tor Lattimore, Andr{\'a}s Gy{\"o}rgy, Pushmeet Kohli\\
\{rayjiang,csilvia,lattimore,agyorgy,pushmeet\}@google.com\\
DeepMind London, UK\\}

\maketitle
\begin{abstract}
Machine learning is used extensively in recommender systems deployed in products. The decisions made by these systems can influence user beliefs and preferences which in turn affect the feedback the learning system receives - thus creating a feedback loop. This phenomenon can give rise to the so-called “echo chambers” or “filter bubbles” that have user and societal implications. In this paper, we provide a novel theoretical analysis that examines both the role of user dynamics and the behavior of recommender systems, disentangling the echo chamber from the filter bubble effect. In addition, we offer practical solutions to slow down system degeneracy. Our study contributes toward understanding and developing solutions to commonly cited issues in the complex temporal scenario, an area that is still largely unexplored.
\end{abstract}

\section{Introduction}
Recommender systems are increasingly used to provide users with personalized product and information offerings \citep{benschafer01e,lu15recommender,covington16deep}. These systems employ user's personal characteristics and past behaviors to generate a list of items that are individually tailored to the user’s preferences.
Whilst extremely successful commercially, there are growing concerns that such systems might lead to a self-reinforcing pattern of narrowing exposure and shift in user's interest, problems that are often referred to in the literature as ``echo chamber'' and ``filter bubble''. A significant amount of research has therefore been devoted to deriving ways to favor diversity in the set of items an individual may be exposed to (see \citet{kunaver17diversity} for a review). However, current understanding of the echo chamber and filter bubble effects is limited and experimental analysis reports conflicting results.

In this paper, we define as echo chamber the effect of a user's interest being positively or negatively reinforced by repeated exposure to a certain item or category of items, thereby generalizing the definition in \citet{sunstein09republic}, where the term is used to refer to over- and limited-exposure to similar political opinions reinforcing one's existing beliefs. We focus the definition of filter bubble introduced by \citet{pariser11filter} to describe just the fact that recommender systems select limited content to serve users online. We provide a theoretical treatment that allows us to consider the echo chamber and filter bubble effects separately. We view user's interest as a dynamical system and treat interest extremes as degeneracy points of the system. We consider different models of dynamics and identify sets of sufficient conditions that make them degenerate over time. We then use this analysis to understand the role played by the recommender system. Finally, we showcase the interplay between the user's dynamics and the recommender system actions in a simulation study using synthetic data and several classic bandit algorithms. The results reveal several pitfalls of recommender system design and point towards mitigation strategies. 

\section{Related Work}
Through an analysis on the MovieLens dataset, \citet{nguyen14exploring} found that the diversity of items recommended, and those users engage with, gets narrower over time. The paper asks whether there is a ``natural'' tendency of degeneration in user interest. Our paper takes steps toward answering this question by providing theoretical conditions for user interest degeneracy.

In the social sciences literature, \citet{flaxman16filter} found that online services are associated with increased political polarization between users as well as increased exposure to the less preferred side of political opinions. Their seemingly counter-intuitive findings are not contradictory according to our results: systems with some level of random exploration can be degenerative. \citet{Barbera15tweeting} also presented evidence of echo chamber related to political issues on Twitter. On the other hand, \cite{Borgesius16should,beam18facebook,Nechushtai18what} found counter-evidence on online news consumption. Another work by \citet{bakshy15exposure} measured the effect of user choices separately from that of the recommendation algorithm, and found that individual choices play a larger role than the algorithm in creating echo chamber on Facebook. This supports our viewpoint that user interests degenerate or not depending on their internal dynamics, the recommender system can only slow down or accelerate the process of degeneration.

\section{Model} \label{sec:method}
 We consider a recommender system that interacts with a user over time\footnote{For simplicity, in this paper we restrict ourselves to the case of a single user, and leave the case of multiple users possibly influencing each-other interests to future work.}.
At every time step $t$, the recommender system serves $l$ items (or categories of items, \eg~news articles, videos, or consumer products) 
to a user from a finite or countably infinite item\footnote{Throughout the paper, ``items'' also mean categories of items.} set $\M$.
In general, the goal of the system is to present items to a user that she is interested in: we assume that, at time step $t$, the user's interest in an item $a \in \M$ is described by a function $\mu_t: \M \to \Reals$ such that $\mu_t(a)$ is large (positive) if the user is interested in the item, and small (negative) if she is not\footnote{Whilst we focus on $\mu_t(a) \in \Reals$, we show in Remark~\ref{thm:rescale}, Appendix~\ref{app:rescale} that our results can be extended to the case where $\mu_t$ belongs to a bounded open interval.}.

Given a recommendation $a_t = (a^1_t,\ldots,a^l_t) \in \M^l$, the user provides some feedback $c_t$ based on her current interests $\mu_t(a^1_t),\ldots,\mu_t(a^l_t)$. This interaction has multiple effects: 
in the traditional literature for recommender systems, the feedback $c_t$ is used to update the internal model $\theta_t$ of the recommender system that has been used to obtain the recommendation $a_t$, and  the new model $\theta_{t+1}$ may depend on $\theta_t$, $a_t$, and $c_t$. In practice $\theta_t$ usually predicts the distribution of user feedback to determine which items $a_t$ should be presented to the user.
In this paper we focus on another effect and consider explicitly that the user's interaction with the recommender system may change her interest in different items for the next interaction, thus the interest $\mu_{t+1}$ may depend on $\mu_t$, $a_t$, and $c_t$. The full model of interaction is depicted in \figref{fig:state_graph}.

\begin{figure}[t]
\begin{center}
\scalebox{0.7}{
\begin{tikzpicture}
\node (theta) [state] {$\theta_{t+1}$};
\path (theta.west)+(-2.5, 0) node (theta0) [state] {$\theta_t$};
\path (theta.east)+(1.5, 0) node (theta1) [txt] {...};
\path (theta.south)+(0, -1) node (at) [state] {$a_{t+1}$};
\path (at.south)+(0, -1.5) node (ct) [state] {$c_{t+1}$};
\path (theta0.west) + (-1.5, 0) node (theta_1) [txt] {...};
\path (theta0.south) + (0, -1) node (at0) [state] {$a_t$};
\path (at0.south)+(0, -1.5) node (ct0) [state] {$c_t$};
\path (ct.south)+(0, -1.0) node (mu) [hstate] {$\mu_{t+1}$};
\path (ct0.south)+(0, -1.0) node (mu0) [hstate] {$\mu_t$};
\path (mu.east)+(1.5, 0) node (mu1) [txt] {...};
\path (mu0.west)+(-1.5, 0) node (mu_1) [txt] {...};

\path (theta_1) + (-1.85, -1.9) node (rec_sys) [txt_color] {Recommender system};
\node[inner sep=0pt] (recsys_pic) at (-6.65, -0.9)
    {\includegraphics[width=.13\textwidth]{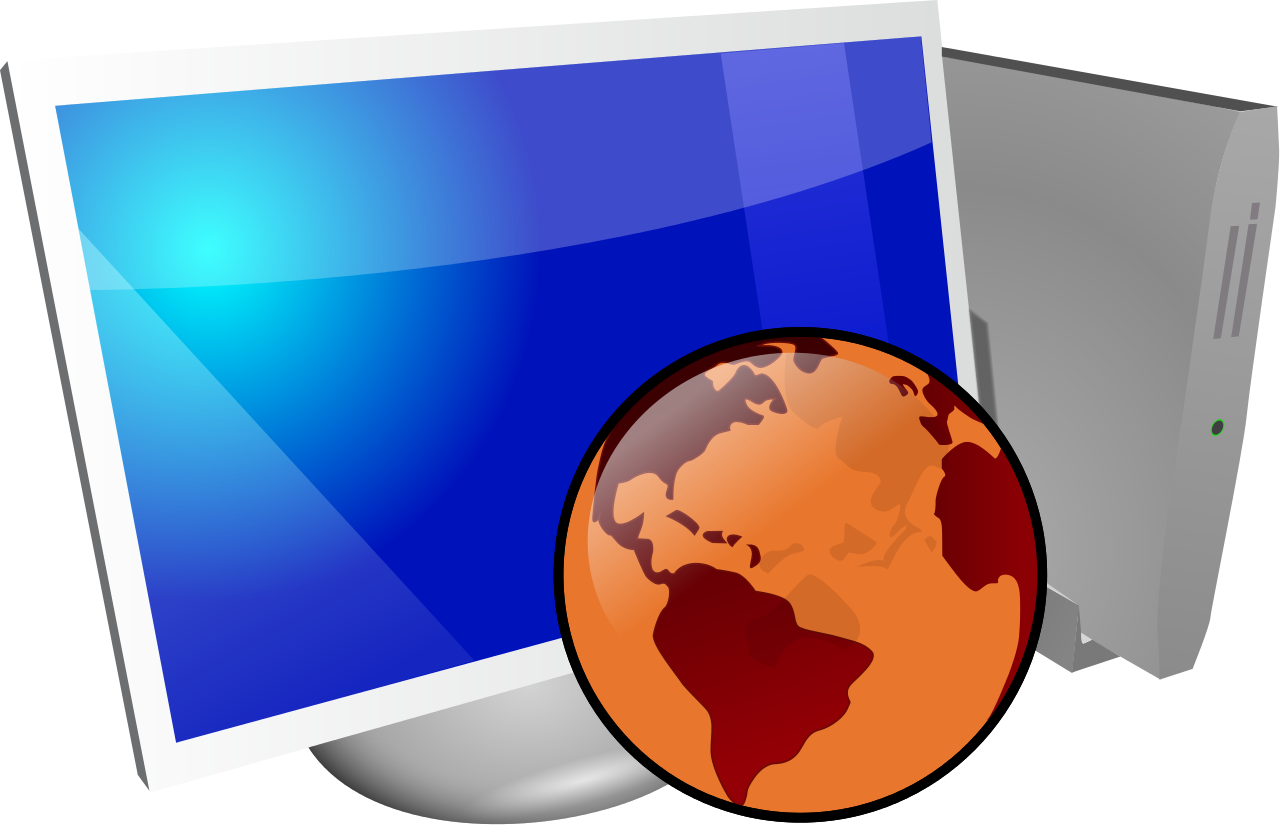}};
\path (theta_1) + (-2, -4.25) node (u_interest) [txt_color] {User};
\node[inner sep=0pt] (user_pic) at (-6.7,-3.15)
    {\includegraphics[width=.1\textwidth]{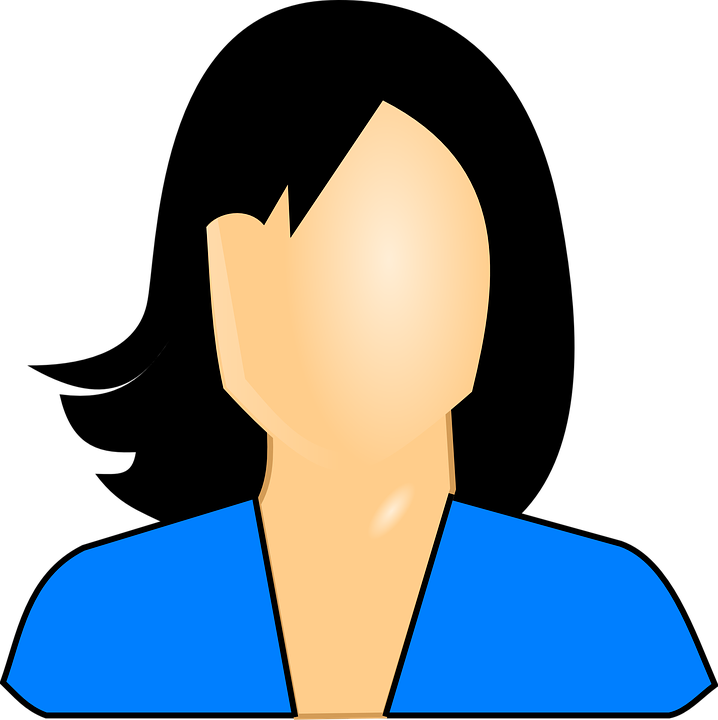}};
\path (theta_1) + (0, -2.15) node (line_start) [txt] {};
\path (theta1) + (0, -2.15) node (line_end) [txt] {};

\path [draw, ->] (theta.south) -- node [above] {} (at.north);
\path [draw, ->] (at.south) -- node [above] {} (ct.north);
\path [draw, ->, dashed] (theta.east) -- node [above] {} (theta1.west);
\path [draw, ->, dashed] (theta0.east) -- node [above] {} (theta.west);
\path [draw, ->] (ct0.east) -- node [above] {} (theta.250);
\path [draw, ->, dashed] (theta_1.east) -- node [above] {} (theta0.west);
\path [draw, ->] (theta0.south) -- node [above] {} (at0.north);
\path [draw, ->] (at0.south) -- node [above] {} (ct0.north);
\path [draw, ->] (mu.north) -- node [above] {} (ct.south);
\path [draw, ->] (mu0.north) -- node [above] {} (ct0.south);
\path [draw, ->] (mu0.east) -- node [above] {} (mu.west);
\path [draw, ->] (mu.east) -- node [above] {} (mu1.west);
\path [draw, ->] (mu_1.east) -- node [above] {} (mu0.west);
\path [draw, ->] (ct0.east) -- node [above] {} (mu.150);
\path [draw, ->] (at0.east) -- node [above] {} (theta.220);
\path [draw, ->] (at0.east) -- node [above] {} (mu.130);

\path [draw, teal, dashed] (line_start.east) -- node [above] {} (line_end.west);
\end{tikzpicture}}

\end{center}
\caption[State graph]{Model of interaction between a recommender system and user over time. Continuous and dashed links indicate existing or possible dependencies, respectively.}
\label{fig:state_graph}
\end{figure}
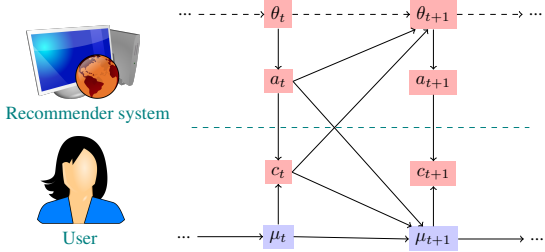

We are interested in studying the evolution of the user's interest. An example of such an evolution is that the interest is reinforced by user interactions with the recommended items, that is, $\mu_{t+1}(a)>\mu_t(a)$ if the user clicks on an item $a$ at time step $t$, while $\mu_{t+1}(a)<\mu_t(a)$ if $a$ is shown but not clicked (here $c_t \in \{0,1\}^l$ can be defined as the indicator vector of clicks to the corresponding items).

To analyze the echo chamber or filter bubble effect, we are interested in understanding when the user’s interest changes extremely, which, in our model, translates to $\mu_t(a)$ taking values arbitrarily different from the initial interest $\mu_0(a)$: large positive values indicate that the user becomes extremely interested in item $a$, while large negative values indicate that the user dislikes $a$.
Formally, for a finite item set $\M$, we can ask if the $L^2$ norm 
$\|\mu_t-\mu_0\|_2 = \left(\sum_{a\in \M} (\mu_t(a)-\mu_0(a))^2\right)^{1/2}$
 can grow arbitrarily large: the user's interest sequence $\mu_t$ is called \emph{weakly degenerate} if
\begin{equation}
    \label{eq:weak}
\limsup_{t\rightarrow \infty}\|\mu_t - \mu_0\|_2 = \infty \quad \text{ \emph{almost surely\footnotemark}}.
\end{equation}
\footnotetext{\ie~ with probably 1.}

A stronger notion of degeneracy, which also requires that once $\mu_t$ drifted away from $\mu_0$ it remains so, is \emph{strong degeneracy}:  the sequence $\mu_t$ is \emph{strongly degenerate} if 
\begin{equation}
    \label{eq:strong}
    \lim_{t\rightarrow \infty}\|\mu_t - \mu_0\|_2 = \infty \qquad \text{\emph{almost surely}}.
\end{equation}

In the next section we show that weak or strong degeneracy occurs under mild sufficient conditions on the evolutionary dynamics of $\mu_t$.

There are multiple ways to extend the above definitions to the case of an infinite item set $\M$. For simplicity, we only consider here replacing $\|\mu_t - \mu_0\|_2$ with $\sup_{a\in\M}|\mu_t(a)-\mu_0(a)|$ in Eqs. (\ref{eq:weak}) and~(\ref{eq:strong}), which is equivalent to the original definitions when $\M$ is finite\footnote{As such, we could have used $\sup_{a\in\M}|\mu_t(a)-\mu_0(a)|$ in our original definitions, but we prefer $\|\mu_t - \mu_0\|_2$ as it also provides some information about the ``average'' deviation of the user's interest over the different items at any finite time $t$.}.

\section{User Interest Dynamics -- Echo Chamber}\label{sec:user}
As items often represent diverse categories of things, we make the simplifying assumption that they are independent from each other. By setting $l=1$ and $a^1_{t}=a$ for all $t$ (\ie, $\M=\{a\}$), we can remove the influence of the recommender system and consider the dynamics of the user's interest separately. This allows us to analyze the echo chamber effect: what happens to the interest $\mu_t(a)$ if item $a$ is served infinitely often (\io).

Since $a$ is fixed, to simplify the notation, we write $\mu_t$ instead of $\mu_t(a)$ in this section. Given $a_t$, according to \figref{fig:state_graph}, $\mu_{t+1}$ is a---possibly stochastic---function of $\mu_t$ (as $\mu_{t+1}$ depends on $c_t$ and $\mu_t$, and $c_t$ depends on $\mu_t$).
Below we discuss the general case when the drift $\mu_{t+1}-\mu_t$ is a nonlinear stochastic function; deterministic models for the drift are considered in Appendix~\ref{app:deterministic}.
\vspace*{-5mm}
\paragraph{Nonlinear Stochastic Model.}
We assume that $\mu_0 \in \Reals$ is fixed and that $\mu_{t+1} = \mu_t + f(\mu_t, \xi_t)$, where
$(\xi_t)_{t=1}^\infty$ is an infinite sequence of independent uniformly distributed random variables that introduce noise into the system (\ie~$\mu_{t+1}$ is a stochastic function of $\mu_t$). The function $f : \Reals \times [0,1]$ is assumed to be measurable, but otherwise arbitrary. Denoting the uniform distribution over $[0,1]$ by $U([0,1])$, let
\begin{align*}
    \bar f(\mu) = \bbE_{\xi \sim U([0,1])}[f(\mu, \xi)]
\end{align*}
be the expected increment $\mu_{t+1} - \mu_t$ 
when $\mu_t = \mu$. We also define 
\begin{align*}
F(\mu, x) = \bbP_{\xi \sim U([0,1])}(f(\mu, \xi) \leq x)
\end{align*}
to be the cumulative distribution of the increment. 
The asymptotic behavior of $\mu_t$
depends on $f$, but under mild assumptions the system degenerates weakly (Theorem~\ref{thm:weak}) or strongly (Theorem~\ref{thm:strong})\footnote{The proofs of these theorems are given in Appendix~\ref{app:proofs}.}.

\begin{theorem}[weak degeneracy]\label{thm:weak}
Assume that $F$ is continuous at $(\mu, 0)$ for all $\mu \in \Reals$ and there exists a $\mu_{\circ} \in \Reals$ such that 1) $F(\mu, 0) < 1$ for all $\mu \geq \mu_{\circ}$, 2) $F(\mu, 0) > 0$ for all $\mu \leq \mu_{\circ}$.
Then the sequence $\mu_t$ is weakly degenerate, \emph{\ie}~$\limsup_{t\to\infty} |\mu_t| = \infty$ almost surely.
\end{theorem}
The assumptions guarantee that within any closed bounded interval there is a constant probability that the random walk escapes to the left/right when starting to the left/right of $\mu_{\circ}$ respectively. Under stronger conditions it is possible to guarantee the divergence of the random walk. We state a simple version of the theorem, but note that the result can be generalized in many ways.

\begin{theorem}[strong degeneracy]\label{thm:strong}
Assume that the conditions of Theorem~\ref{thm:weak} hold,
and additionally that there exists $c \in \Reals$ such that $|\mu_{t+1} - \mu_t| \leq c$ almost surely and there exists an $\epsilon > 0$ such that for all sufficiently large $\mu$ it holds that
$\bar f(\mu) > \epsilon$, and for all sufficiently small $\mu$ it holds that $\bar f(\mu) < -\epsilon$. 
Then $\lim_{t\to\infty} \mu_t = \infty$ or $\lim_{t\to\infty} \mu_t = -\infty$ almost surely.
\end{theorem}

Intuitively, weak degeneracy occurs in a stochastic environment if the user's interest has some non-zero probability of drifting up when above some threshold, and of drifting down when below. Strong degeneracy holds if additionally $|\mu_{t+1}-\mu_t|$ is bounded and for $\mu_t$ sufficiently large/small the increment $\mu_{t+1}-\mu_t$ has positive/negative drift that is larger than a constant.

Theorems~\ref{thm:weak} and~\ref{thm:strong} show that the user's interest degenerates under very mild conditions, in particular, in the model we consider in our simulation studies. Thus, in such cases degeneracy can only be avoided if an item (or item category) is showed only finitely many times; otherwise one can only hope to control how fast $\mu_t$ degenerates (\ie~tends to $\infty$).

\section{System Design Role -- Filter Bubble}\label{sec:recsys}
In the previous section we discussed conditions for degeneracy for different user interest dynamics. In this section  we examine the other side of the story, the influence of recommender system actions in creating filter bubbles. 
We typically do not know the dynamics of the user's interest in the real world. However, we consider the relevant scenario to the echo chamber/filter bubble problem where user's interest in some items has degenerative dynamics, and examine how to design a recommender system that slows down the degeneracy process. We consider three dimensions, namely model accuracy, amount of exploration, and growing candidate pool. 

\vspace*{-3.5mm}
\paragraph{Model Accuracy.}

One common goal of recommender systems designers is to increase the prediction accuracy of the internal model $\theta_t$. 
How does model accuracy coupled with greedy optimal $a_t$ affect the speed of degeneration? We examine this question for the extreme case of exact predictions, \ie~$\theta_t = \mu_t$, we call such a prediction model the \emph{oracle model}. We argue that under the {\it surfacing assumption} explained below, the oracle model coupled with greedily optimal action selection results in the quickest degeneracy. 

In order to analyze the problem concretely, we focus on the degenerate linear dynamics model for $\mu_t(a)$ for $a \in \M$, \ie~ $\mu_{t+1}(a) = (1+k) \mu_t(a) + b$. Then we can solve for $\mu_t(a)$, obtaining 
$$\mu_t(a) = (\mu_0(a) + \frac{b(a)}{k(a)})(1+k(a))^t - \frac{b(a)}{k(a)}\,,$$ for $|1+k(a)| > 1$ (see Appendix~\ref{app:deterministic}).

\emph{Surfacing Assumption:}
Let $[m]=\{1,2,\ldots,m\}$ be the candidate set of size $m$. If a subset of items $\s\subset[m]$ leads to positive degenerate dynamics (\ie~$\mu_t(a)\rightarrow +\infty$ for all $a\in\s$), then we assume that there exists a time $\tau > 0$ such that, for all $t \geq \tau$, $\s$ takes up the top $|\s|$ items in terms of values of $\mu_t$, sorted by the base value of the exponential function, $|1+k(a)|$.

The surfacing assumption makes sure that the quickest degenerating items surface out to the top list given enough time of exposure. It can be generalized to nonlinear stochastic dynamics of $\mu_t$ provided that the items from $\s$ have an almost surely stable ordering of degeneracy speed $|\mu_t(a) - \mu_0(a)|/t$ over time.

Under the general surfacing assumption, after time $\tau$, the quickest way to degeneration is to serve the top $l$ items according to $\mu_t$, or $\theta_t$ of the oracle model. Even if the assumption is violated to some degree, the oracle model still leads to degeneracy very efficiently by picking the top $l$ items according to $\mu_t$ which are likely to receive positive feedback due to high $\mu_t$, and therefore increasing $\mu_{t+1}$ and reinforcing the past choices.

In practice the recommender system models are inaccurate. We can think of inaccurate models as the oracle model with different levels of noises added to $\theta_t$. We discuss inaccurate models in the next section. 

\vspace*{-3.5mm}
\paragraph{Amount of Exploration.}
Consider a type of $\epsilon$-random exploration where $a_t$ always picks the top $l$ items out of a finite candidate pool $[m]$ with uniform $\epsilon$ noise on $\theta_t$, \ie~according to $\theta'_t = \theta_t + U([-\epsilon, \epsilon])$.

Given the same model sequence $\theta_t$, the bigger $\epsilon$ is, usually the slower the system degenerates. However, in practice $\theta_t$ is learned from observations, and the random exploration added to an oracle model may in fact accelerate degeneration: random exploration can help reveal the most positively degenerating items over time making the surfacing assumption more likely to be true (we show this phenomenon in the simulation experiments below, \figref{fig:noise_oracle}). In addition, if user interests have degenerative dynamics, even recommending items uniformly at random leads to degeneration, albeit quite slowly.

How do we then make sure that the recommender system does not make user interests degenerate? One way is to limit the number of times an item for which the user's interest dynamics is degenerative is served to the user. In practice it is hard to detect which items correspond to degenerative dynamics, however we can generally prevent degeneration if all items are served only a finite number of times, which suggests having an ever growing pool of candidate items.
\begin{figure}[t!]
\vspace{-3.5mm}
\centering
\subfloat[][{\it Optimal Oracle}, $\mu_t$\label{fig:OptimalOracle_mu}]{\includegraphics[height=1.1in,trim={1cm 0 0 0},clip]{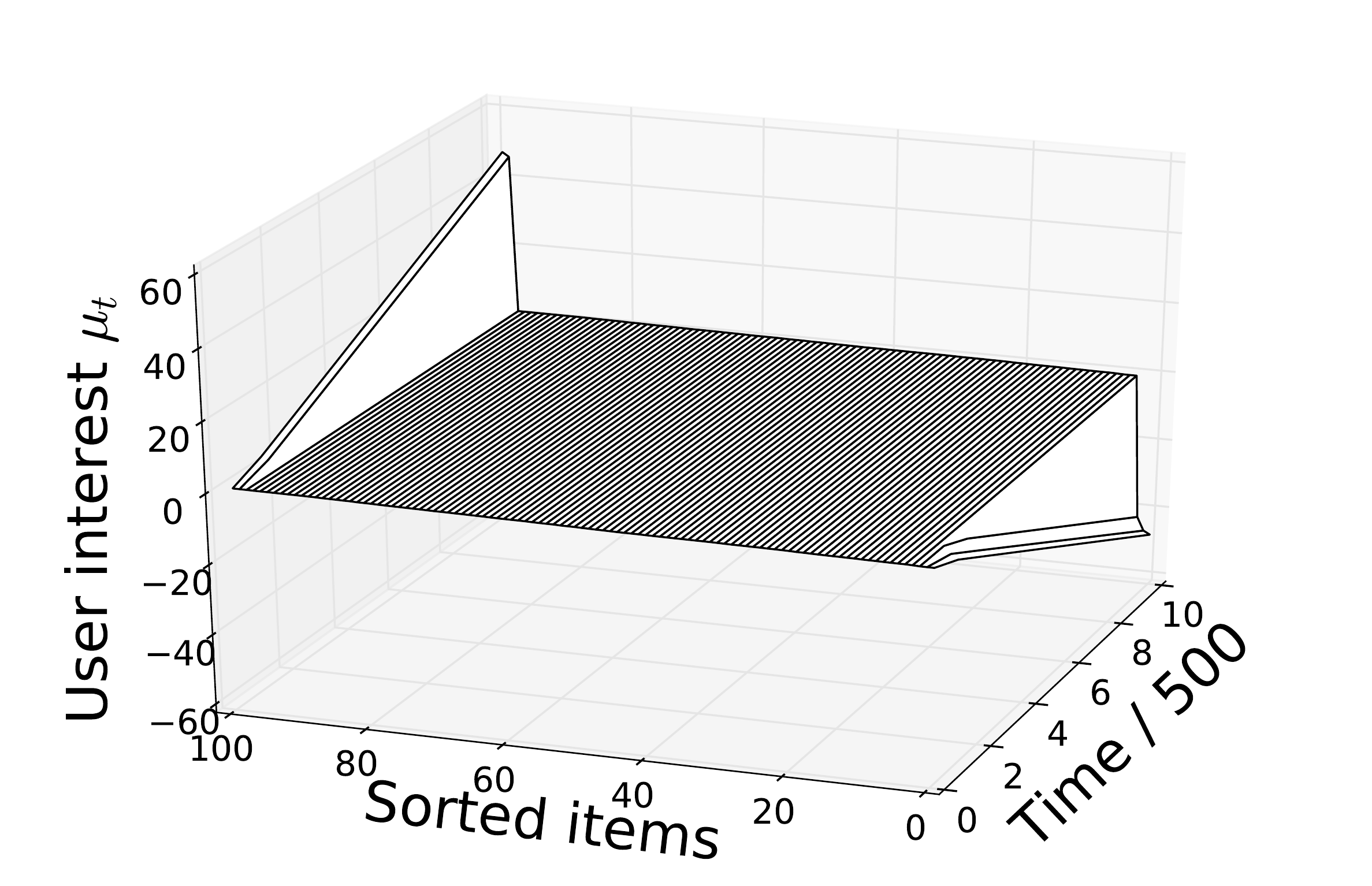}}
\subfloat[][{\it Optimal Oracle}, serving rate\label{fig:OptimalOracle_a}]{\includegraphics[height=1.1in,trim={1cm 0 0 0},clip]{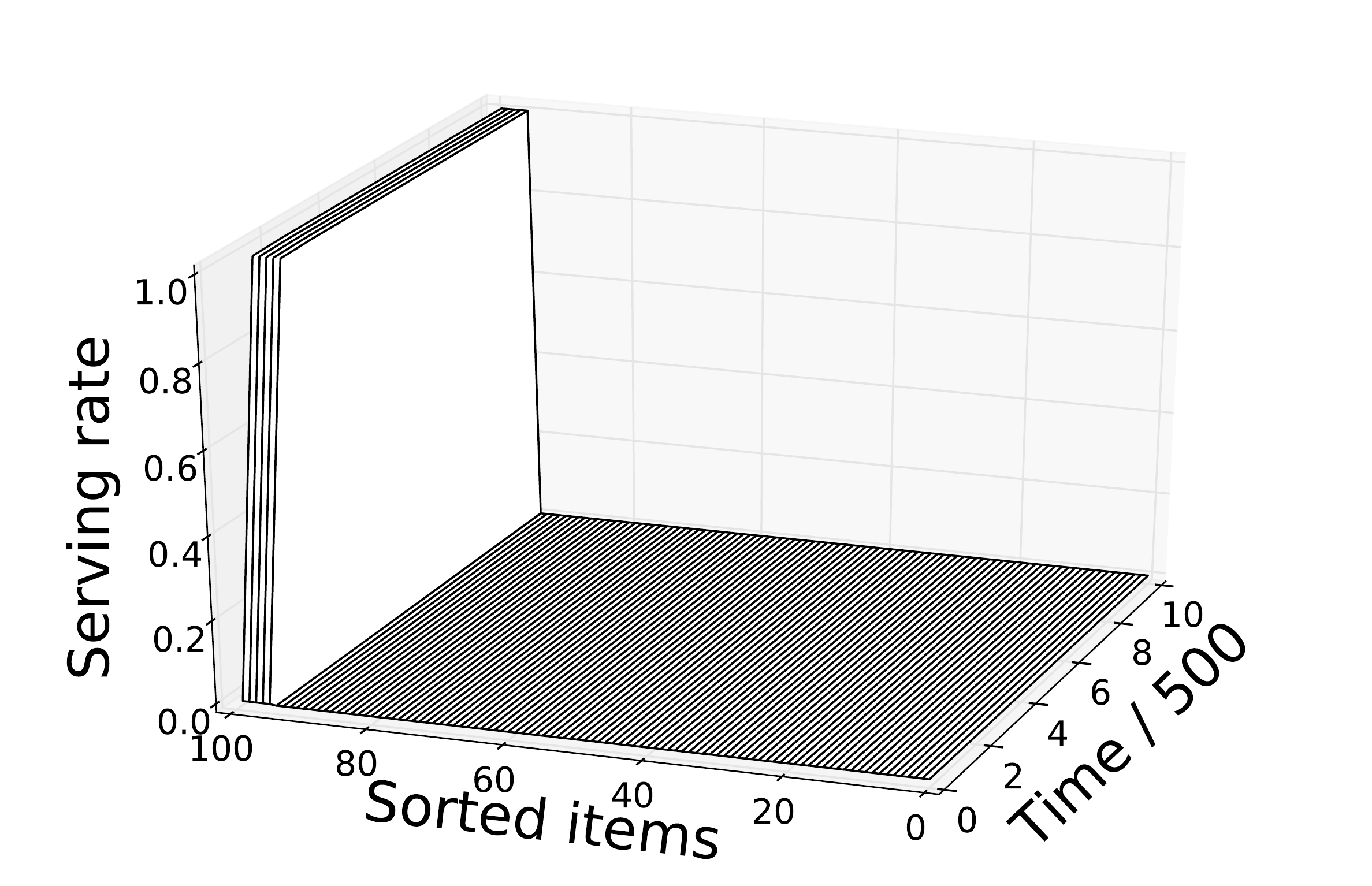}}
\hfill
\subfloat[][{\it Oracle}, $\mu_t$\label{fig:Oracle_mu}]{\includegraphics[height=1.1in,trim={1cm 0 0 0},clip]{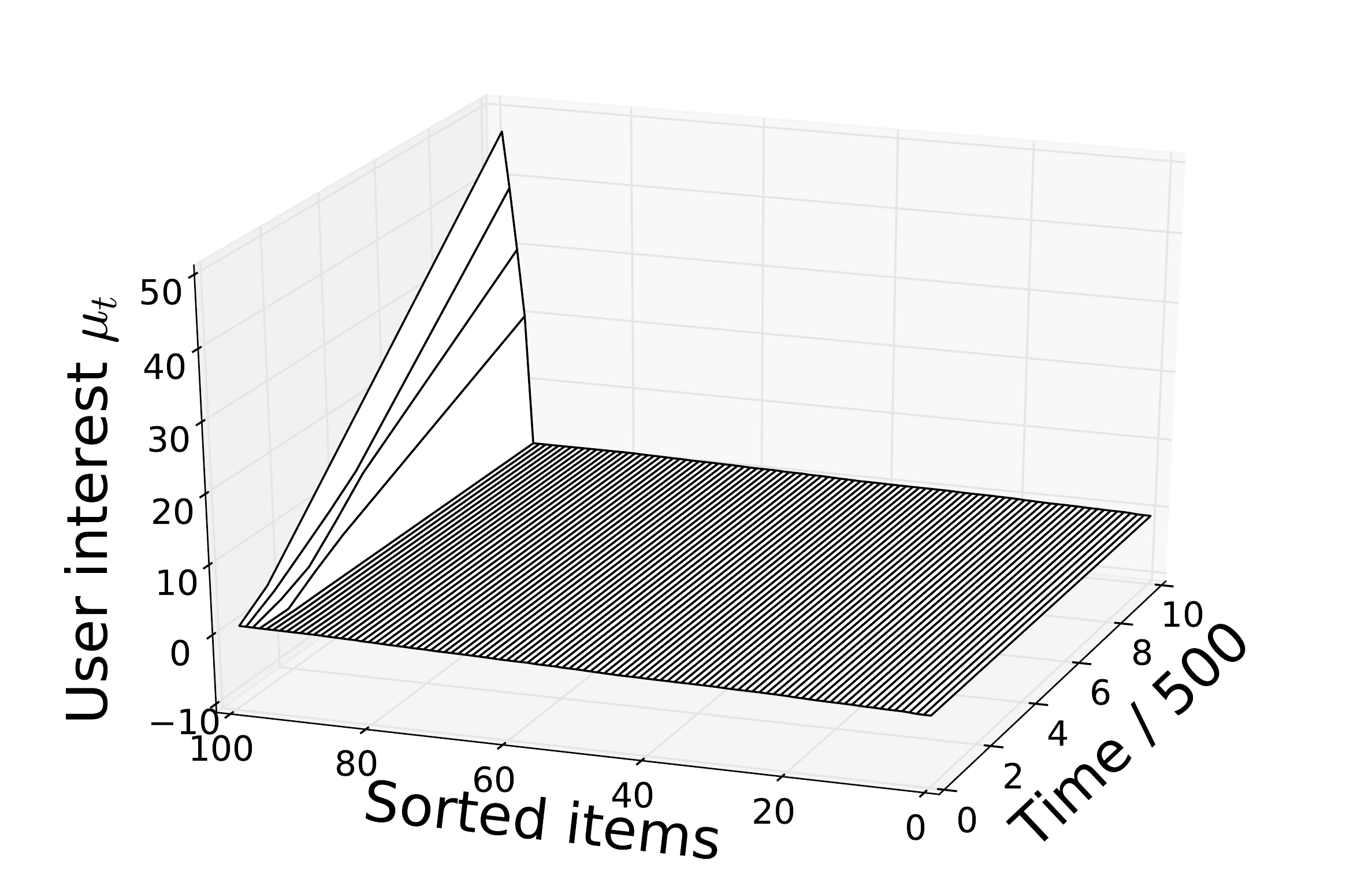}}
\subfloat[][{\it Oracle}, serving rate\label{fig:Oracle_a}]{\includegraphics[height=1.1in,trim={1cm 0 0 0},clip]{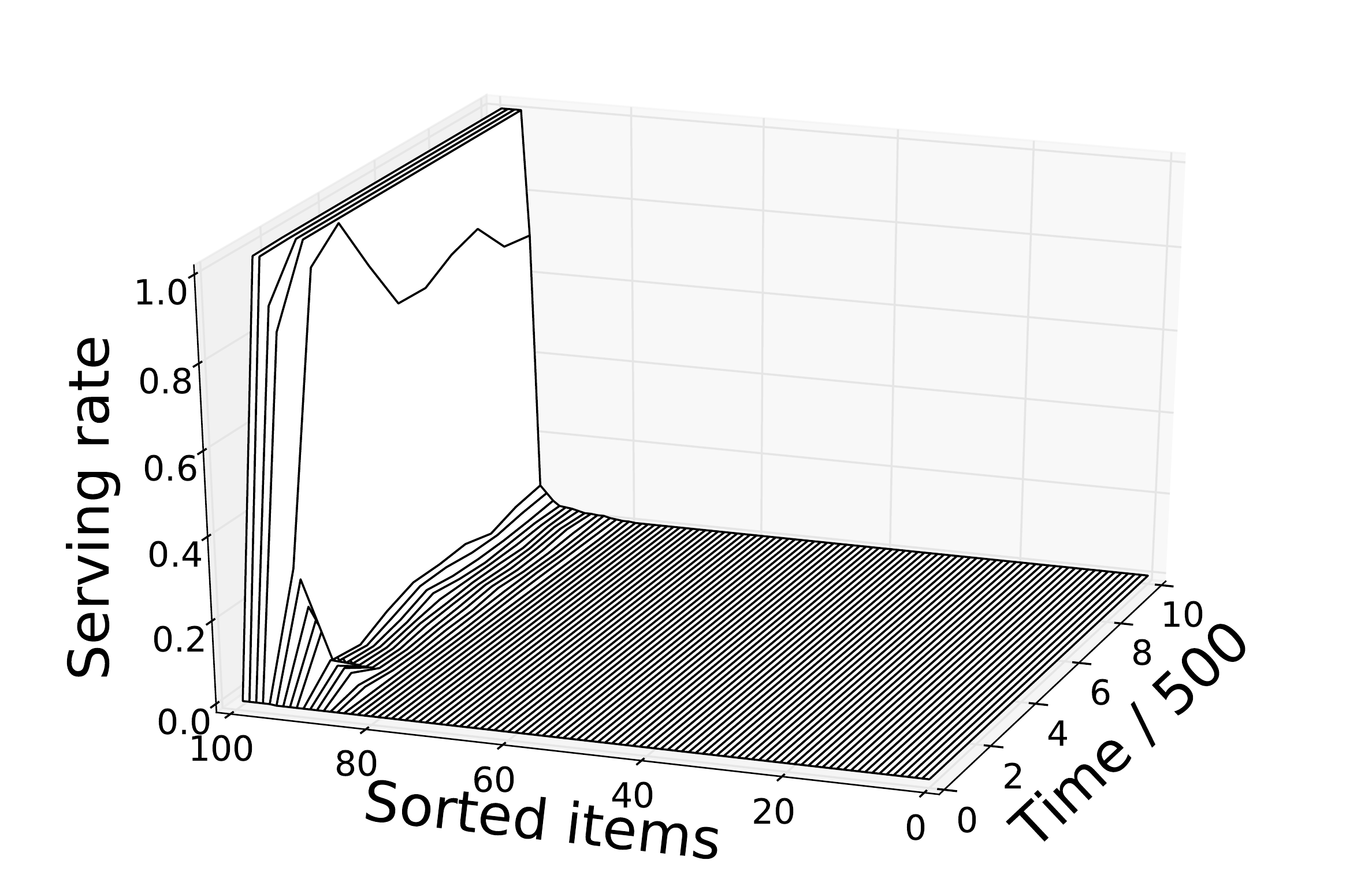}}
\hfill
\subfloat[][{\it TS}, $\mu_t$\label{fig:ThompsonSampling_mu}]{\includegraphics[height=1.in,trim={1cm 0 0 0},clip]{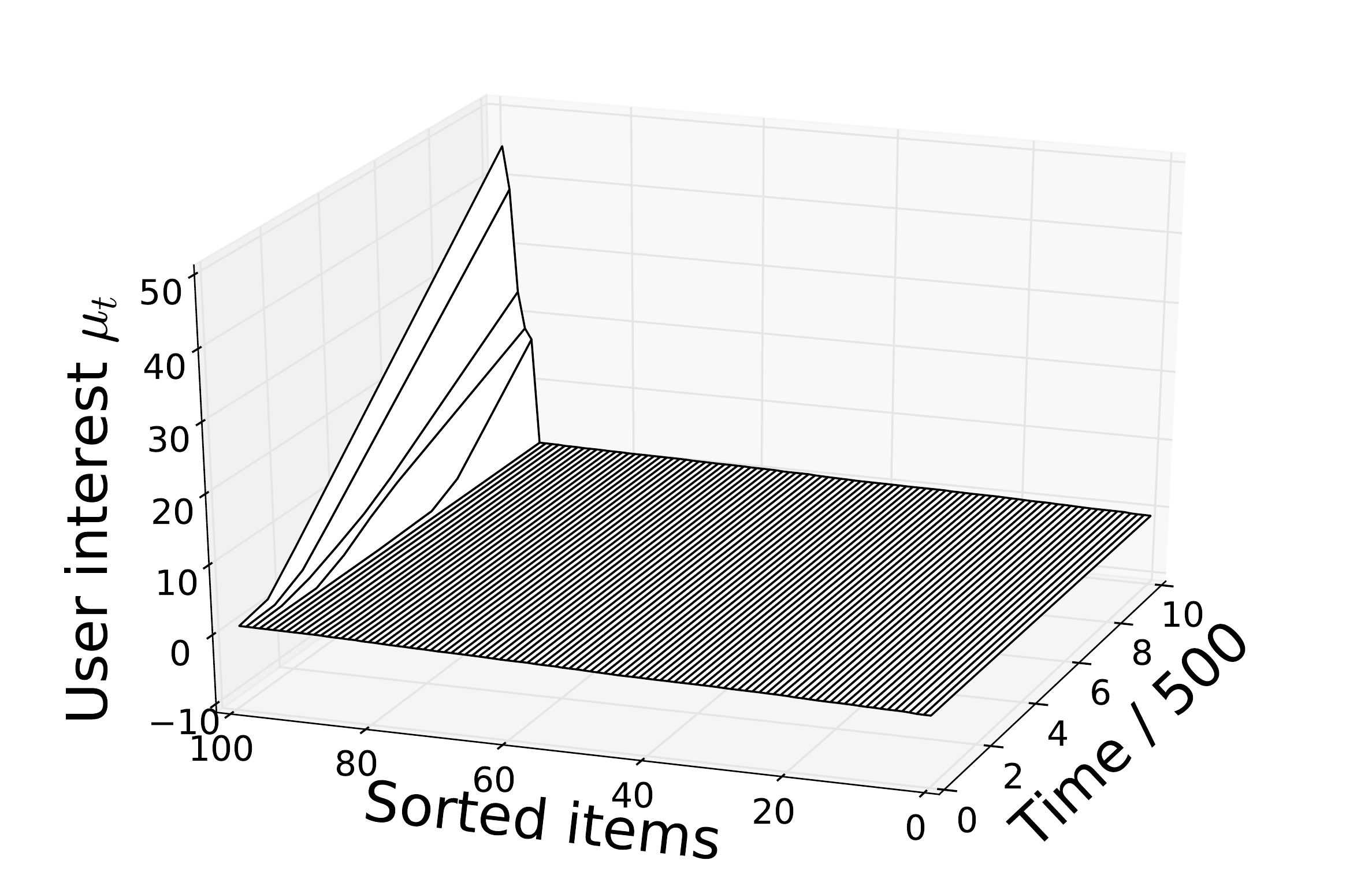}}
\subfloat[][{\it TS}, serving rate\label{fig:ThompsonSampling_a}]{\includegraphics[height=1.in,trim={1cm 0 0 0},clip]{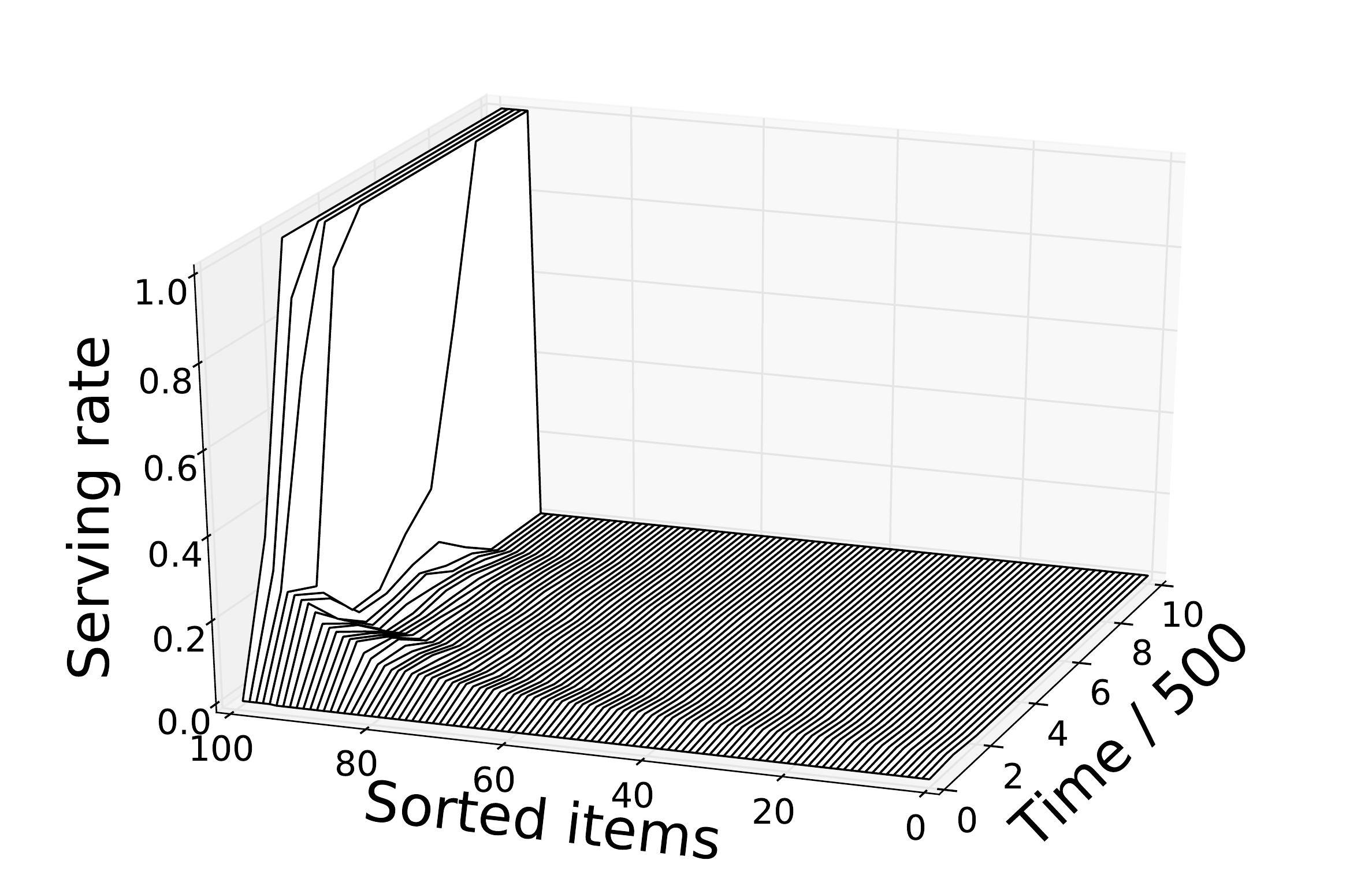}}
\hfill
\subfloat[][{\it UCB}, $\mu_t$\label{fig:UCB_mu}]{\includegraphics[height=1.1in,trim={1cm 0 0 0},clip]{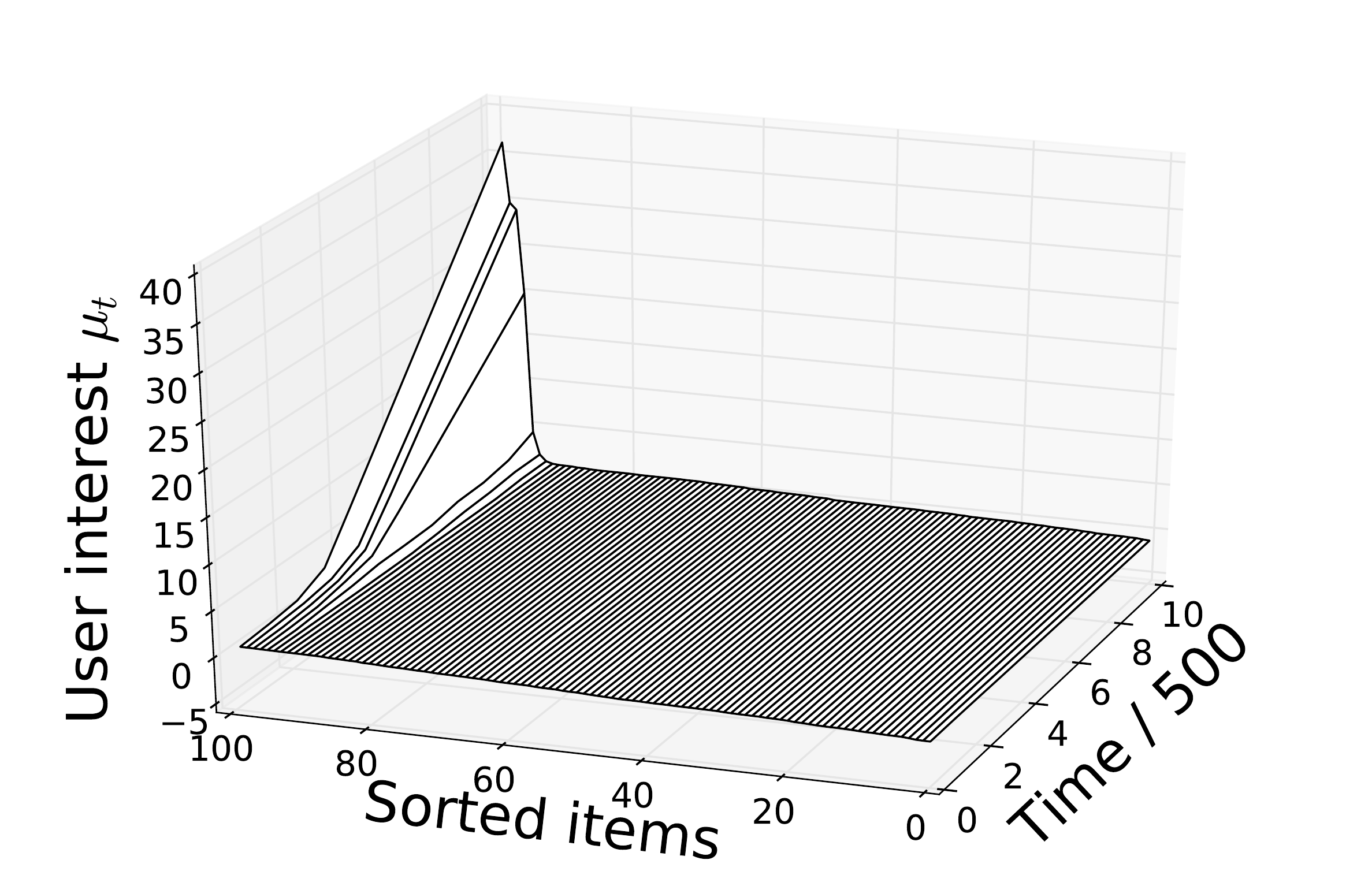}}
\subfloat[][{\it UCB}, serving rate\label{fig:UCB_a}]{\includegraphics[height=1.1in,trim={1cm 0 0 0},clip]{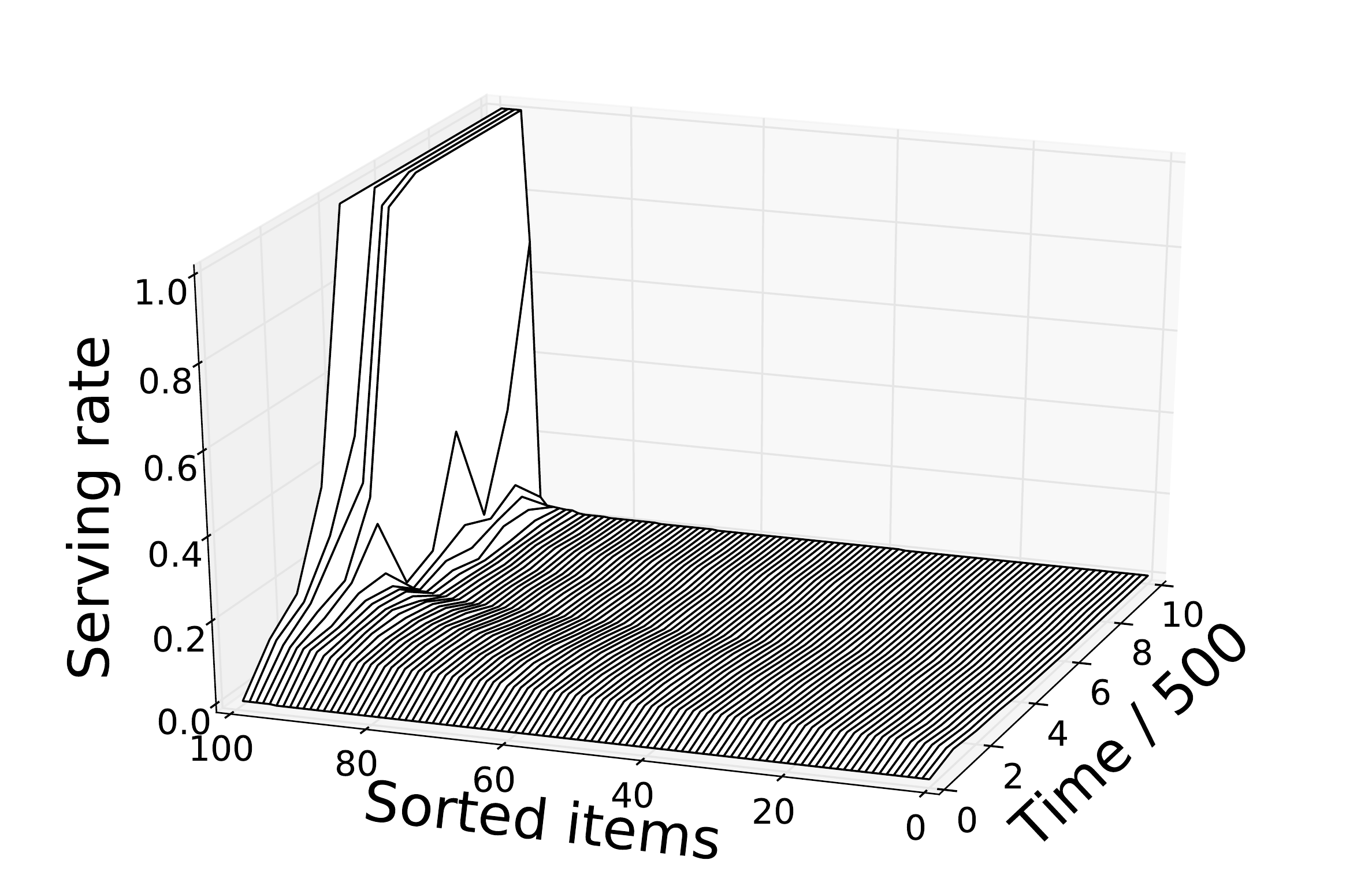}}
\hfill
\subfloat[][{\it Random}, $\mu_t$\label{fig:random_mu}]{\includegraphics[height=1.1in,trim={1cm 0 0 0},clip]{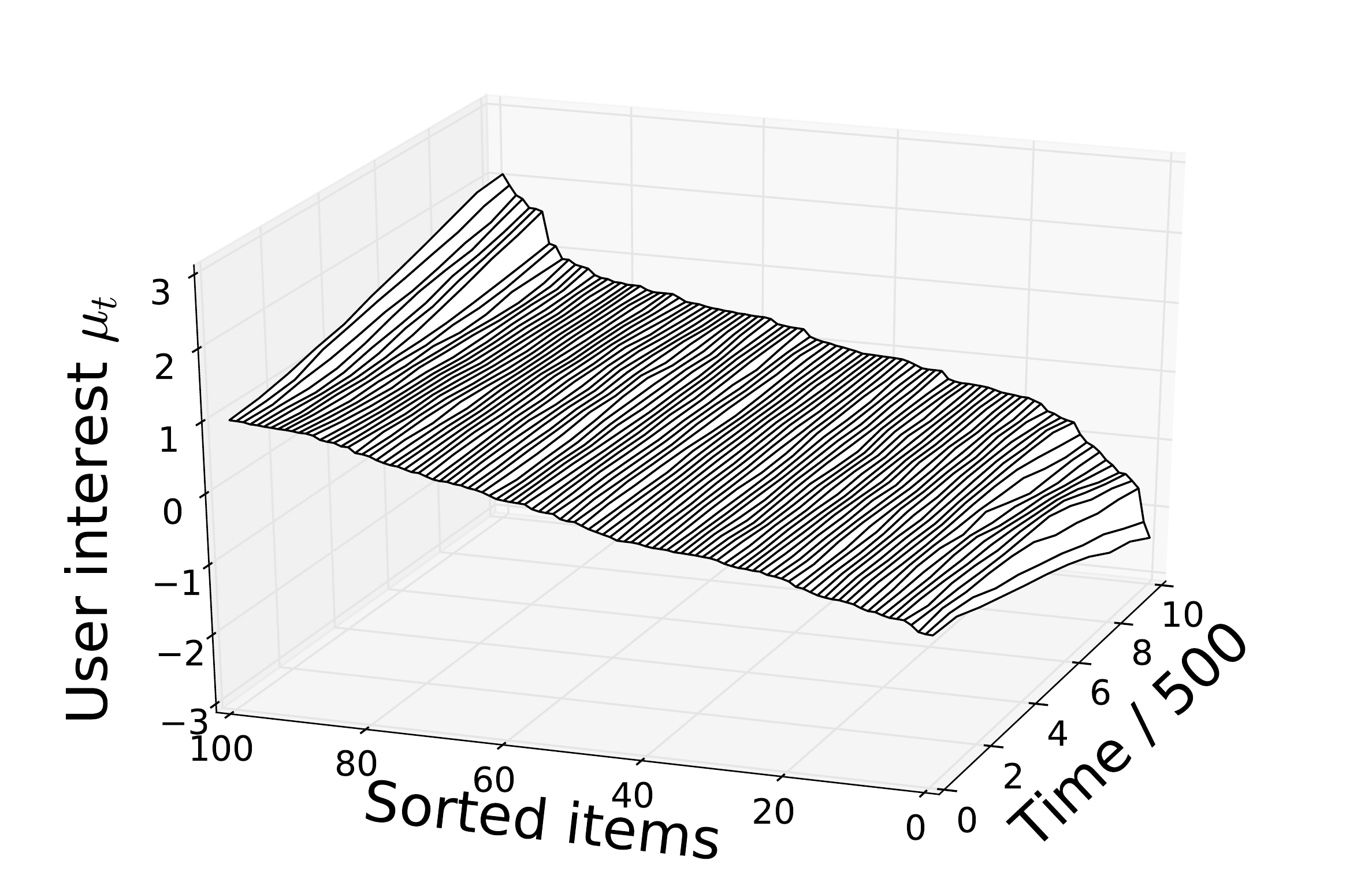}}
\subfloat[][{\it Random}, serving rate\label{fig:random_a}]{\includegraphics[height=1.1in,trim={1cm 0 0 0},clip]{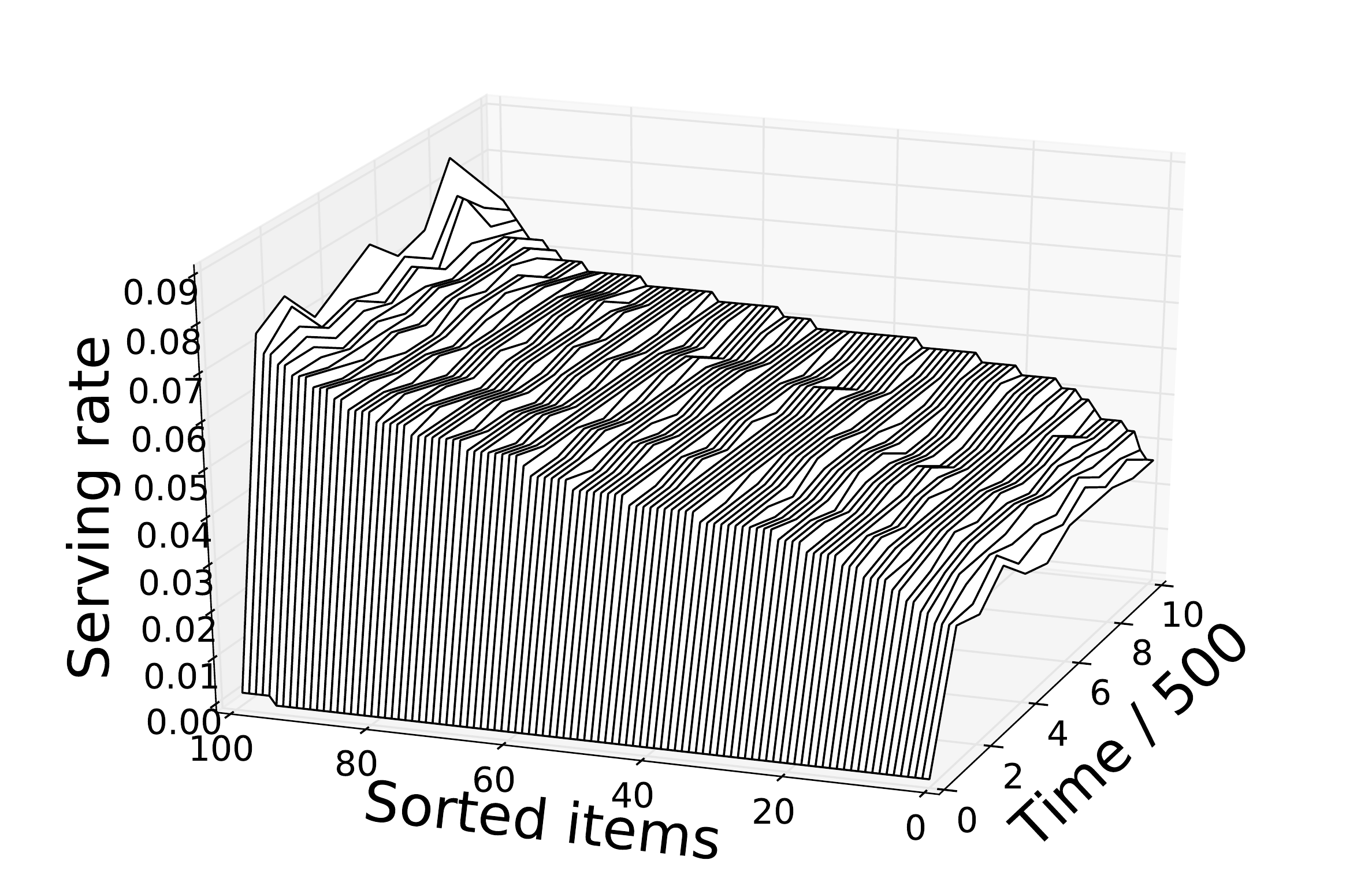}}
\caption[]{Echo chamber and filter bubble effect for {\it Optimal Oracle, Oracle, TS, UCB} and {\it Random} models. Sorted user interest $\mu_t$ and serving rates are plotted every 500 steps. Under all models except for the {\it Random Model}, very quickly both the top items served and the top user interests narrow down to the ($l=$) 5 most positively reinforced items.}
\label{fig:sim_filter_bubble}
\end{figure}

\vspace*{-3.5mm}
\paragraph{Growing Candidate Pool $\M$.}
With a growing candidate pool, at every time step an additional set of new items becomes available to be served to the user. Hence the domain of the function $\mu_t$ expands as $t$ increases. Adding new items at least linearly often 
is a necessary condition to avoid possible degeneration, since in a finite or any sublinearly growing candidate pool, by the pigeon hole principle, there must exist at least one item that is served \io, which is degenerate in the worst case scenario (also under general conditions described \eg~in Theorem~\ref{thm:strong}). However, with an at least linearly growing candidate pool $\M$ the system can potentially impose the maximum number of times any item is served to a user and prevent degeneration.

\section{Simulation Experiments}\label{sec:simulation}
In this section, we consider a simple degenerative dynamics for $\mu_t$ and examine degeneration speed under five different recommender system models. We further demonstrate that adding new items to the candidate pool can be an effective solution against system degeneracy.

We create a simulation for the model of interaction between a recommender system and a user of \figref{fig:state_graph}. Consider a possibly growing candidate pool of items of initial size $m_0$ and of size $m_t$ at time step $t$. At each time step $t$, a recommender system picks the top $l$ out of the $m_t$ items $a_t = (a_t^1, \ldots, a_t^l)$ according to the internal model $\theta_t$ to serve to a user. The user considers each of the $l$ items independently and chooses to click on a (possibly empty) subset of them, thereby generating a binary vector $c_t$ of size $l$ where $c_t(a_t^i)$ gives the user feedback on item $a_t^i$, according to $c_t(a_t^i) \sim Bernoulli(\phi(\mu_t(a_t^i)))$, where $\phi$ is the sigmoid function $\phi(x)=1/(1+e^{-x})$. The system then updates the model $\theta_{t+1}$ based on the past actions, feedbacks and the current model parameter $\theta_t$. We assume that the user's interest increases/decreases by $\delta(a')$ if the item $a'$ receives/does not receive a click, \ie~
\begin{align}
    &\mu_{t+1}(a_t^i) - \mu_t(a_t^i) =  
    \begin{cases}
    \delta(a_t^i)  & \text{ if } c_t(a_t^i)=1,\\
    -\delta(a_t^i)  & \text{ otherwise, }
    \end{cases}
\end{align}
where the function $\delta$ maps from the candidate set to $\mathbb{R}$. From Theorem~\ref{thm:strong}, we know that $\mu_t\rightarrow\pm\infty$ for every item. 
In the experiment, we set $l=5$ and 
sample the drift $\delta$ from a uniform random distribution $U([-0.01, 0.01])$. The user's initial interest $\mu_0$ for all items is independently sampled from a uniform random distribution $U([-1, 1])$. 

\begin{figure}[t]
\centering
\includegraphics[height=1.9in]{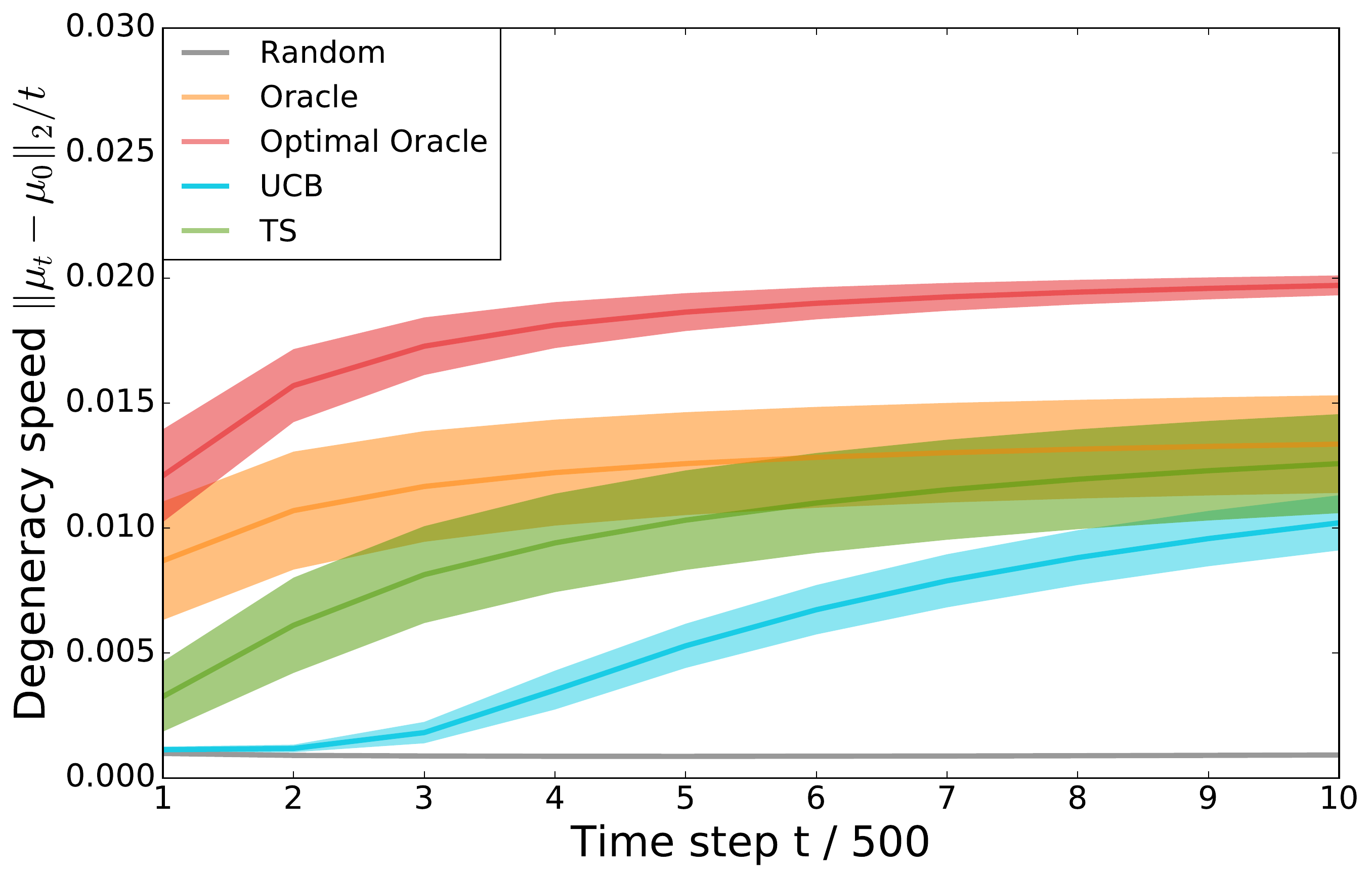}
\caption[]{System evolves for 5,000 time steps with report interval 500. The results are averaged over 30 runs with the shaded area indicating the standard deviation. In terms of the degeneracy speed, {\it Optimal Oracle} $>$ {\it Oracle} $>$ {\it TS} $>$ {\it UCB} $>$ {\it Random}.}
\label{fig:m_100}
\end{figure}
\begin{figure}[t]
\centering
\subfloat[][Degeneracy rate vs candidate pool size vs time \label{fig:changing_m_T}]{\includegraphics[height=1.5in,trim={0.5cm 1.5cm 0 3cm},clip]{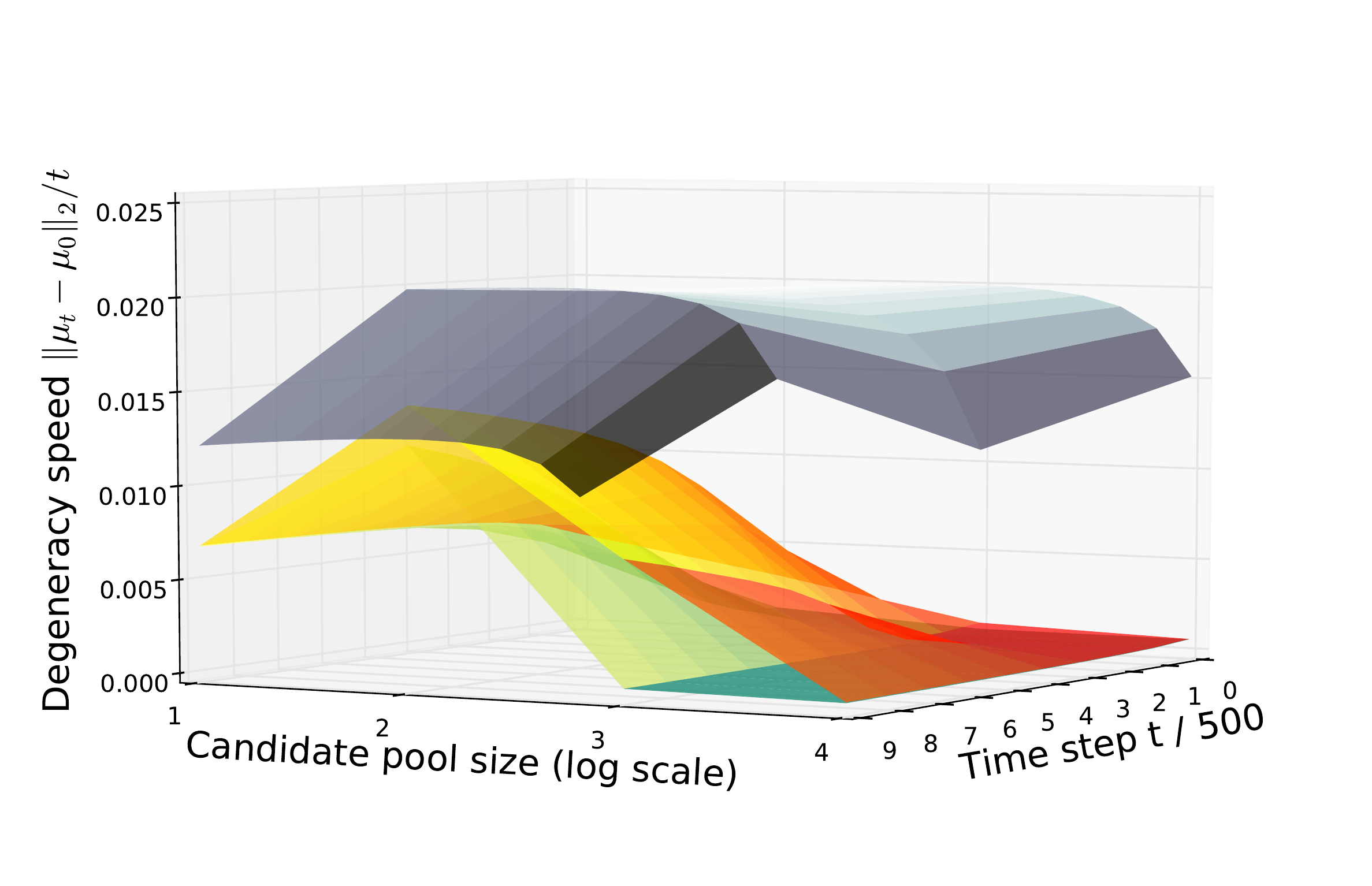}}
\hfill
\subfloat[][Degeneracy rate vs candidate pool size\label{fig:changing_m}]{\includegraphics[height=1.9in]{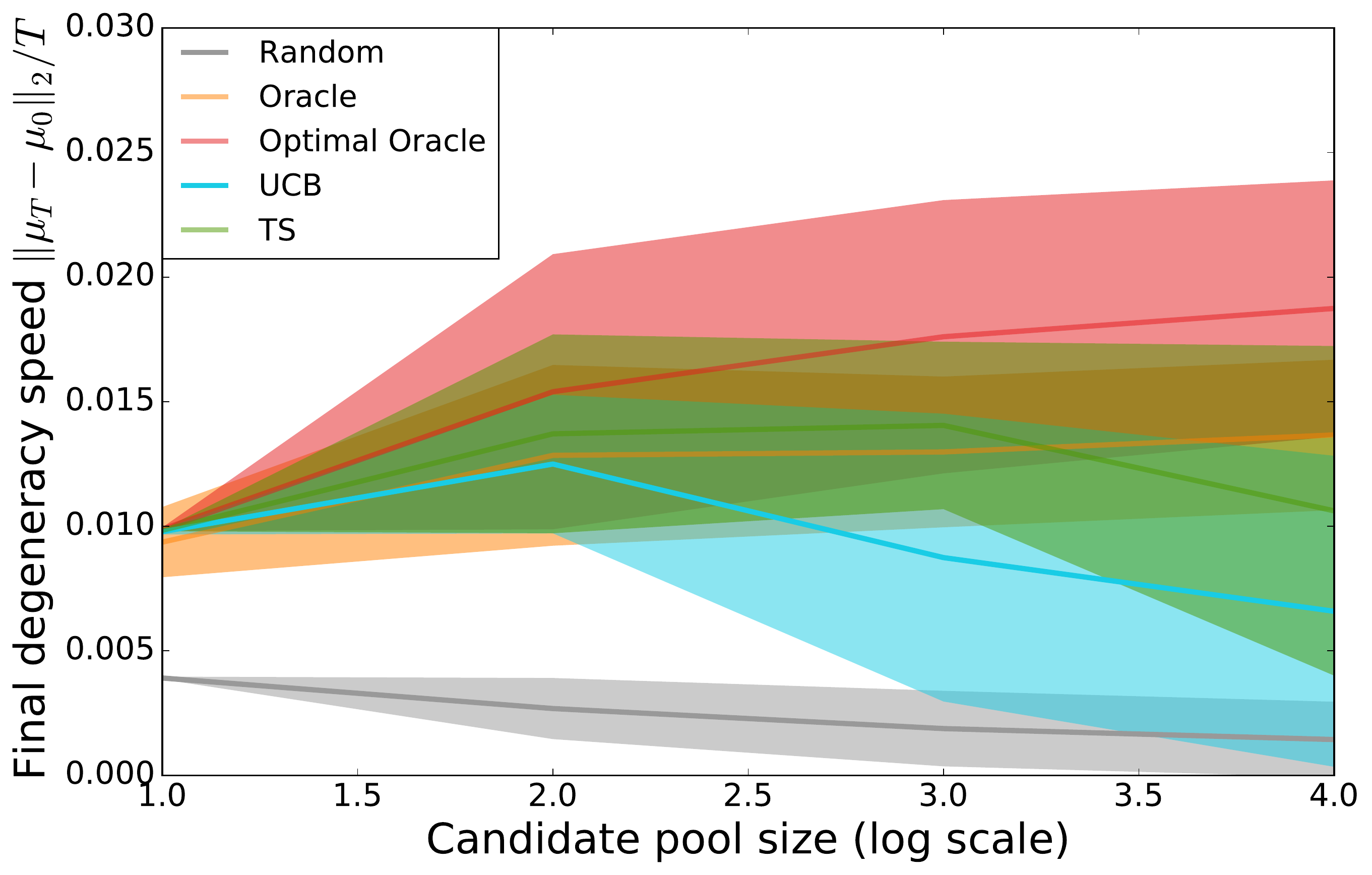}}
\caption[]{
\subref{fig:changing_m_T} Degeneracy surfaces for {\it Optimal Oracle} (grey), {\it UCB} (green) and {\it TS} (orange) up to time $T=5,000$ while varying candidate pool sizes $\log_{10} m = 1, 2, 3, 4$. A larger candidate pool requires a longer time for exploration for the bandit algorithms, but among the three models {\it UCB} slows down system degeneracy the most given a large candidate pool. 
\subref{fig:changing_m} Degeneracy speeds at $T=20,000$ of {\it Optimal Oracle} and the {\it Oracle} are higher given a larger size of the candidate set, but those of the {\it and Random Model, UCB}, and {\it TS} are lower.}
\label{fig:fixed_pool}
\end{figure}
\begin{figure}[t]
\centering
\includegraphics[height=1.9in]{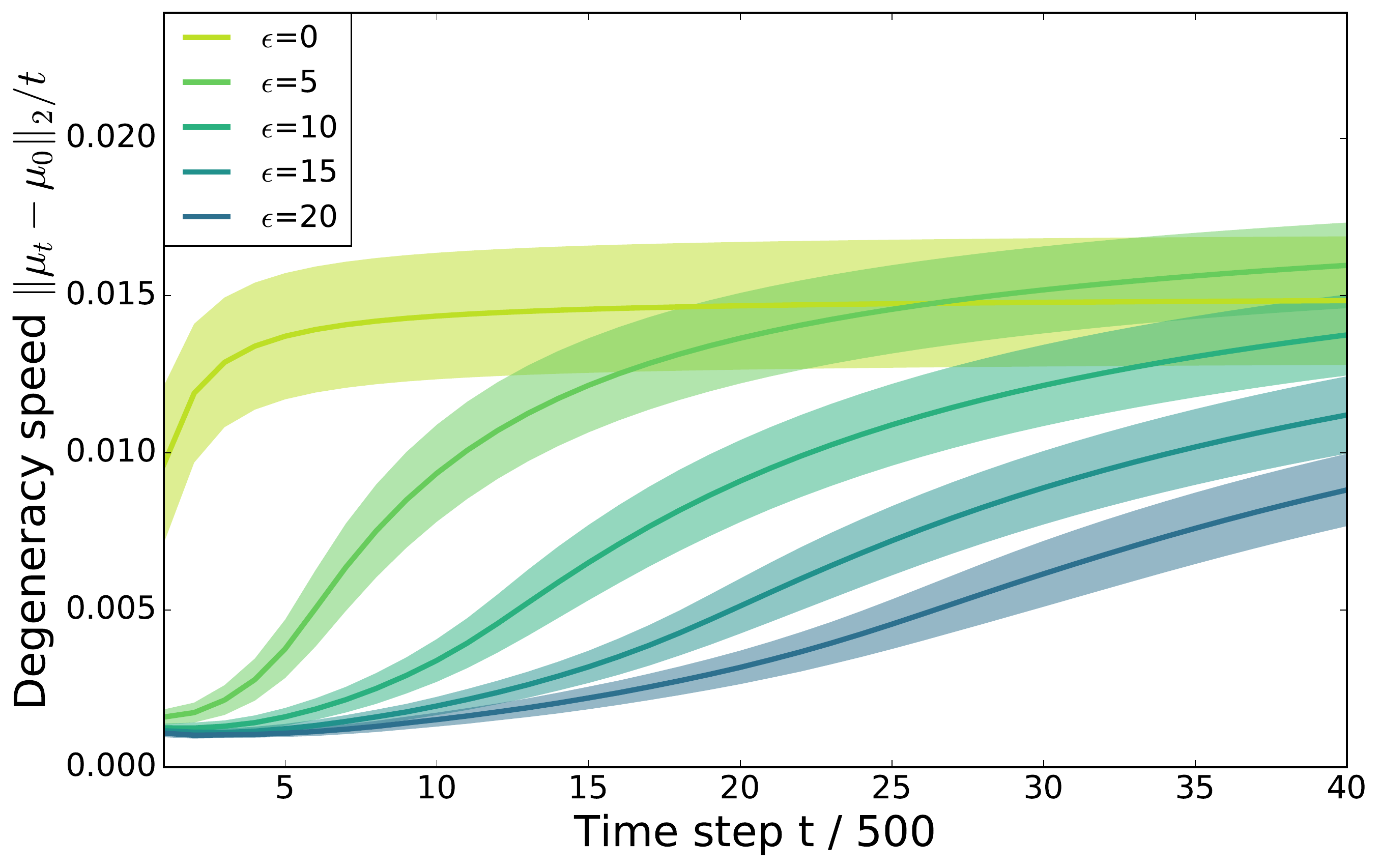}
\caption[]{Degeneracy speed for the {\it Oracle} model with different noise levels $\epsilon\in[0, 10]$ up to $T=20,000$. Adding noise to {\it Oracle} ($\epsilon=0$) accelerates degeneration but as the noise level grows, degeneracy slows down.}
\label{fig:noise_oracle}
\end{figure}
The internal recommender system model is updated according to following five algorithms:
\begin{itemize}
\item {\bf Random Model:}
Instead of picking top items, the set of items $a_t = (a_t^1, \ldots, a_t^l)$ is sampled from a uniform random distribution over the candidate set $U([m_t])$.
\item {\bf Oracle:}
$\theta_{t+1}(a_t^i) = \mu_{t+1}(a_t^i)$, $\forall i$. 
\item {\bf Optimal Oracle:}
$\theta_{t+1}(a_t^i) = \delta(a_t^i)$, $\forall i$. 
This model does not pick the highest $l$ items according to $\mu_t$ but according to $\delta$. Thus, it always picks the fastest degenerating items, therefore  maximizing both the long term user engagement $\sum_{t=0}^\infty \|c_t\|_1$ and the degeneracy speed. For a fixed candidate pool $\M$, this model is equivalent to an \emph{Oracle} that satisfies the Surfacing Assumption. 
\item {\bf Upper Confidence Bound Multi-armed Bandit Algorithm ({\it UCB})} \citep{lai87adaptive,auer02finite,lattimore19bandit}:
We use the version of UCB algorithm in Chapter 8 of \citet{lattimore19bandit}, however most UCB algorithms perform similarly to the purpose of this experiment. The algorithm prioritizes serving any item from the candidate set that has never been served before. This treatment includes the initial $m_0$ items as well as later whenever new items are added to the candidate pool. At time step $t$, UCB serves $l' (0\leq l'\leq l)$ previously unserved items and the top $l-l'$ items according to values of $\theta_t$. Define $f(t) = 1 + t\log^2(t)$ and we use the following model update
$\theta_{t+1} (a) = \hat{c}_t(a) + \sqrt{2\log f(t)/T_a(t)}$, where $\hat{c}_t$ is the empirical average of feedbacks on item $a$, \ie~$\hat{c}_t(a) = \sum_{0 \le i \le t, a\in a_i}^t c_t(a)/T_a(t)$, and $T_a(t)$ is the number of times item $a$ has been served up to time $t$, i.e. $T_a(t) = \sum_{0 \le i \le t, a\in a_i} 1$. 
\item {\bf Thompson Sampling Multi-armed Bandit Algorithm ({\it TS})} \citep{thompson33on}:
We initialize $\alpha_0(a)=1, \beta_0(a)=1$ for any new item $a$. If $a$ is served at time $t$, we perform the update $\alpha_{t+1}(a) = \alpha_t(a) + c_t(a), \beta_{t+1}(a) = \beta_t(a) + 1 - c_t(a)$. At any time $t$, the internal model $\theta_t$ is sampled from the corresponding beta distribution $Beta(\alpha_t, \beta_t)$.
\end{itemize}

\subsection{Echo Chamber \& Filter Bubble Effect}
We examine the echo chamber and filter bubble effects by running the simulation on a candidate pool of fixed size $m_t=m = 100$ with time horizon $T=5,000$. 

In \figref{fig:sim_filter_bubble}, we show the degeneration of user interest $\mu_t$ (left column) and the serving rate (right column) of every item as each recommender model evolves in time. The serving rate of an item shows how often it is served within the report interval. In order to see the distribution clearly, we sort the items according to the z-values at the report time. Although all models cause user interest degeneration, the degeneration speeds are quite different ({\it Optimal Oracle} $>$ {\it Oracle, TS, UCB} $>$ {\it Random Model}). The {\it Oracle, TS} and {\it UCB} optimize based on $\mu_t$ and so we see a positive degenerative dynamics for $\mu_t$. The {\it Optimal Oracle} optimizes on the degeneration speed directly and not on $\mu_t$ so we see both a positive and negative degeneration in $\mu_t$. The {\it Random Model} also drifts $\mu_t$ in both directions, but at a much slower rate. However, overall except for the {\it Random Model}, very quickly both the top items served and the top user interests narrow down to the ($l=$) 5 most positively reinforced items. 

\subsection{Speed of Degeneracy}
Next, we compare the degeneracy speed for the five recommender system models on both fixed and growing candidate sets. As the $L^2$ distance that measures system degeneracy is asymptotically linear for all five models (see Appendix~\ref{app:linear_speed}), we quantify degeneracy speeds by compare empirically $\|\mu_t-\mu_0\|_2 / t$ in finite candidate pools for different experiment setups. 

Figure~\ref{fig:m_100} shows the degeneracy speed of five models averaged across 30 runs when we take $m=100$ and evolve the system for $T=5,000$ steps. We see that the {\it Optimal Oracle} results in the fastest degeneration by far, followed by the {\it Oracle, TS} and {\it UCB}. The {\it Random Model} offers the slowest degeneracy speed. 

\vspace*{-3mm}
\paragraph{The Effect of Candidate Pool Size.} In \figref{fig:changing_m_T} we compare the {\it Optimal Oracle}, {\it UCB} and {\it TS}' degeneracy speed $\|\mu_t-\mu_0\|_2 / t$ up to 5,000 time steps and candidate pool sizes $m = 10, 10^2, 10^3, 10^4$. Apart from the {\it Random} model, we see that {\it UCB} slows down system degeneracy the most given a large candidate pool since it is forced to explore any unserved item first. A larger candidate pool requires a longer time for exploration for the bandit algorithms. As the candidate pool size grows to 10,000 {\it UCB}'s degeneracy speed never peaks up given the time horizon, but will eventually grow given a longer time. {\it TS} has higher degeneracy speed due to weaker exploration on new items. The {\it Optimal Oracle} accelerates degeneration given a larger pool, as it can pick potentially faster degenerative items than from a smaller pool. 

Additionally, in \figref{fig:changing_m} we plot all five models degeneracy speed for $T=20,000$ against the same changing candidate pool sizes. The degeneracy speed of the {\it Optimal Oracle} and the {\it Oracle} increases with the size of the candidate set, but that of the {\it and Random Model, UCB}, and {\it TS} decreases. In practice, having a large candidate pool can be a temporary solution to slow down system degeneration.

\vspace*{-3mm}
\paragraph{The Effect of the Noise Level.} 
Next we show the influence of internal model inaccuracy on degeneracy speed. We compare the {\it Oracle} model with different amounts of uniformly random noises, \ie~the system serves the top $l$ items according to the noisy internal model $\theta'_t = \theta_t + U([-\epsilon, \epsilon])$. The candidate pool has fixed size $m=100$. In \figref{fig:noise_oracle}, we vary $\epsilon$ from 0 to 10. Counter-intuitively adding noise to {\it Oracle} accelerates degeneration since faster degenerative items may be selected by chance than those fixed set of top $l$ items ranked by $\mu_0$, and more likely satisfies the Surfacing Assumption. Given $\epsilon>0$, as expected, we see a nice monotonically increasing damping effect on degeneracy speed as the noise level grows.

\begin{figure}[t]
\centering
\hspace*{-.4cm}  
\includegraphics[height=1.6in]{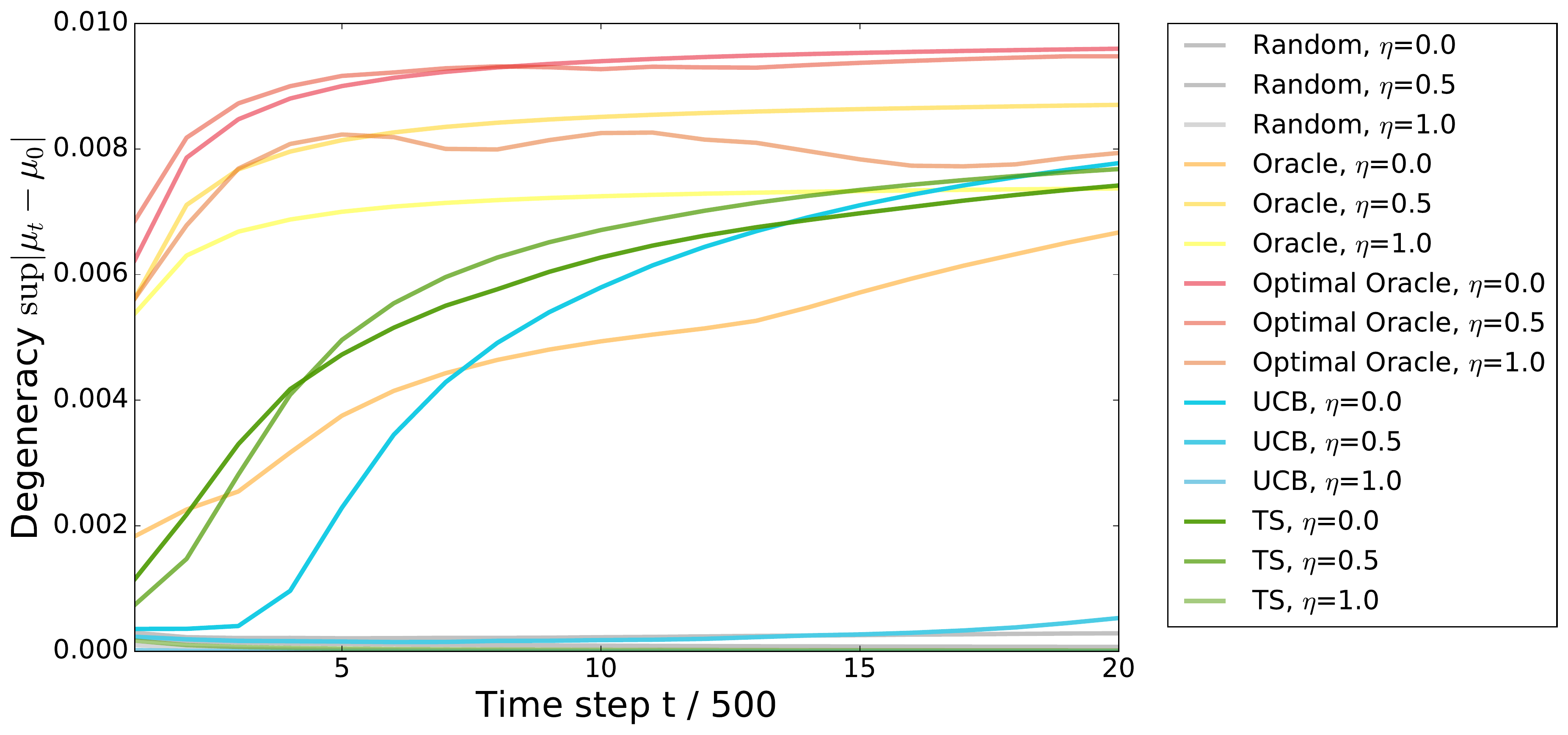}
\caption[]{Comparison of the five models with growing candidate pools at different rates $\eta=0, 0.5, 1.0$, degeneracy up to $T=10,000$, averaged over 10 runs. Both the {\it Oracle} and the {\it Optimal Oracle} for all growth rates are degenerate. The {\it Random Model} and {\it UCB} stop degeneration at sublinear growth while {\it TS} model requires linear growth to stop degeneration.}
\label{fig:growing_pool}
\end{figure}

\vspace*{-3mm}
\paragraph{Growing Candidate Pool.} 
We extend the definition of degeneracy speed to an infinite candidate pool by computing $\sup_{a\in\M}|\mu_t(a)-\mu_0(a)| / t$ (see Appendix~\ref{app:linear_speed} for an asymptotic analysis). Since the degeneracy speed may not be asymptotically linear for all five models, we examine directly the sup distance $\sup_{a\in\M}|\mu_t(a)-\mu_0(a)|$ over 10,000 time steps. To construct growing candidate pools at different growth speed, we define a growth function $m_t = \left\lfloor m_0 + l t^{\eta} \right\rfloor$ by varying the growth parameter\footnote{$\eta=0$ gives a fixed candidate pool, $0<\eta<1$ gives sub-linear growth, $\eta=1$ gives linear growth.} $\eta=0, 0.5, 1$, where $m_0=100$. In \figref{fig:growing_pool} we average the results over 10 independent runs. Both the {\it Oracle} and the {\it Optimal Oracle} for all growth rates are degenerate. The {\it Random Model} stops degeneration at sublinear growth, $\eta=0.5$, so does {\it UCB} thanks to forced exploration on previously unserved items, although its trajectory has a small upward tilt. The {\it TS} model degenerates at sublinear growth but stops degeneration at linear growth $\eta=1$. For all models, the higher the growth rate $\eta$, the slower they degenerate, if they do at all. Overall when applicable, an ideally linearly growing candidate set and continuous random exploration seem to be good remedies against an adversarial dynamics of $\mu_t$ to best prevent degeneracy.  

\section{Conclusion}
We provided a theoretical analysis of the echo chamber and filter bubble effects for recommender systems. We used the dynamical system framework to model user's interest and treated interest extremes as degeneracy points of the system. We gave formal definitions of system degeneracy and provided sufficient conditions which make the system degenerate with both deterministic and stochastic dynamics. On the recommender system side, we discussed the influence on degeneracy speed of three independent factors in system design, \ie~model accuracy, amount of exploration, and the growth rate of the candidate pool. An oracle model often leads to quick degeneracy of the system, while continuous exploration and a large candidate pool size can help slow it down. The best remedies against system degeneracy we found are continuous random exploration and growing the candidate pool at least linearly. 

Our work has two main limitations. First, since user interests are hidden variables that are not directly observed, a good measure or proxy for user interests is necessary in practice to study degeneration reliably. Second, we assumed that items and users are independent from each other -- we will extend the theoretical analysis to the case of possibly mutually dependent items and users in a future work.

\section{Acknowledgments}
We would like to thank William Isaac, Michael Mathieu, Krishnamurthy Dvijotham, Timothy Mann and Dilan Gorur, for helpful discussions and advice.

\bibliography{aies}

\begin{thebibliography}{}

\bibitem[\protect\citeauthoryear{Auer, Cesa-Bianchi, and
  Fischer}{2002}]{auer02finite}
Auer, P.; Cesa-Bianchi, N.; and Fischer, P.
\newblock 2002.
\newblock Finite-time analysis of the multiarmed bandit problem.
\newblock {\em Machine Learning} 47(2):235--256.

\bibitem[\protect\citeauthoryear{Bakshy, Messing, and
  Adamic}{2015}]{bakshy15exposure}
Bakshy, E.; Messing, S.; and Adamic, L.~A.
\newblock 2015.
\newblock Exposure to ideologically diverse news and opinion on facebook.
\newblock {\em Science} 348:1130--1132.

\bibitem[\protect\citeauthoryear{Barber\'{a} \bgroup et al\mbox.\egroup
  }{2015}]{Barbera15tweeting}
Barber\'{a}, P.; Jost, J.~T.; Nagler, J.; Tucker, J.~A.; and Bonneau, R.
\newblock 2015.
\newblock Tweeting from left to right: Is online political communication more
  than an echo chamber?
\newblock {\em Psychological Science} 26(10):1531--1542.

\bibitem[\protect\citeauthoryear{Beam, Hutchens, and
  Hmielowski}{2018}]{beam18facebook}
Beam, M.~A.; Hutchens, M.~J.; and Hmielowski, J.~D.
\newblock 2018.
\newblock Facebook news and (de)polarization: reinforcing spirals in the 2016
  us election.
\newblock {\em Information, Communication \& Society} 21(7):940--958.

\bibitem[\protect\citeauthoryear{{Ben Schafer}, Konstan, and
  Riedl}{2001}]{benschafer01e}
{Ben Schafer}, J.; Konstan, J.; and Riedl, J.
\newblock 2001.
\newblock E-commerce recommendation applications.
\newblock {\em Data Mining and Knowledge Discovery}  115--153.

\bibitem[\protect\citeauthoryear{Borgesius \bgroup et al\mbox.\egroup
  }{2016}]{Borgesius16should}
Borgesius, F. J.~Z.; Trilling, D.; M\"{o}ller, J.; Bodó, B.; de~Vreese, C.~H.;
  and Helberger, N.
\newblock 2016.
\newblock Should we worry about filter bubbles?
\newblock {\em Internet Policy Review}.

\bibitem[\protect\citeauthoryear{Covington, Adams, and
  Sargin}{2016}]{covington16deep}
Covington, P.; Adams, J.; and Sargin, E.
\newblock 2016.
\newblock Deep neural networks for youtube recommendations.
\newblock In {\em Proceedings of the 10th ACM Conference on Recommender Systems
  (RecSys)}.

\bibitem[\protect\citeauthoryear{Flaxman, Goel, and
  Rao}{2016}]{flaxman16filter}
Flaxman, S.; Goel, S.; and Rao, J.
\newblock 2016.
\newblock Filter bubbles, echo chambers, and online news consumption.
\newblock {\em Public Opinion Quarterly} 80:298--320.

\bibitem[\protect\citeauthoryear{Galor}{2007}]{galor07discrete}
Galor, O.
\newblock 2007.
\newblock {\em Discrete Dynamical Systems}.
\newblock Springer.

\bibitem[\protect\citeauthoryear{Kunaver and Porl}{2017}]{kunaver17diversity}
Kunaver, M., and Porl, T.
\newblock 2017.
\newblock Diversity in recommender systems a survey.
\newblock {\em Knowledge-Based Systems} 123(C):154--162.

\bibitem[\protect\citeauthoryear{Lai}{1987}]{lai87adaptive}
Lai, T.~L.
\newblock 1987.
\newblock Adaptive treatment allocation and the multi-armed bandit problem.
\newblock {\em Ann. Statist.} 15(3):1091--1114.

\bibitem[\protect\citeauthoryear{Lattimore and
  Szepesv{\'a}ri}{2019}]{lattimore19bandit}
Lattimore, T., and Szepesv{\'a}ri, C.
\newblock 2019.
\newblock {\em Bandit Algorithms}.
\newblock Cambridge (to appear).

\bibitem[\protect\citeauthoryear{Lu \bgroup et al\mbox.\egroup
  }{2015}]{lu15recommender}
Lu, J.; Wu, D.; Mao, M.; Wang, W.; and Zhang, G.
\newblock 2015.
\newblock Recommender system application developments: A survey.
\newblock 74.

\bibitem[\protect\citeauthoryear{Nechushtai and Lewis}{2018}]{Nechushtai18what}
Nechushtai, E., and Lewis, S.~C.
\newblock 2018.
\newblock What kind of news gatekeepers do we want machines to be? filter
  bubbles, fragmentation, and the normative dimensions of algorithmic
  recommendations.
\newblock {\em Computers in Human Behavior}.

\bibitem[\protect\citeauthoryear{Nguyen \bgroup et al\mbox.\egroup
  }{2014}]{nguyen14exploring}
Nguyen, T.; Hui, P.; Harper, F.; Terveen, L.; and Konstan, J.
\newblock 2014.
\newblock Exploring the filter bubble: The effect of using recommender systems
  on content diversity.
\newblock In {\em Proceedings of the 23rd international conference on World
  wide web},  677--686.

\bibitem[\protect\citeauthoryear{Pariser}{2011}]{pariser11filter}
Pariser, E.
\newblock 2011.
\newblock {\em The Filter Bubble: What the Internet is Hiding from You}.
\newblock Penguin UK.

\bibitem[\protect\citeauthoryear{Sunstein}{2009}]{sunstein09republic}
Sunstein, C.~R.
\newblock 2009.
\newblock {\em Republich.com 2.0}.
\newblock Princeton University Press.

\bibitem[\protect\citeauthoryear{Thompson}{1933}]{thompson33on}
Thompson, W.~R.
\newblock 1933.
\newblock On the likelihood that one unknown probability exceeds another in
  view of the evidence of two samples.
\newblock {\em Biometrika} 25(3/4):285--294.

\end{thebibliography}
\bibliographystyle{aaai}

\begin{appendices}
\setcounter{secnumdepth}{1}  

\section{Proofs} \label{app:proofs}
\begin{proof}[Proof of Theorem~\ref{thm:weak}]
Let $\bbP_\mu$ denote the measure carrying $(\mu_t)_{t=1}^\infty$ when $\mu_0 = \mu$.
Assume without loss of generality 
that $\mu_{\circ} = 0$ and let $B > 0$ be
arbitrary. The result will follow by showing that for all
$\mu \in [-B,B]$ it holds that
\begin{align}
    \bbP_\mu(\text{exists } t : \mu_t \notin [-B, B]) = 1\,.
    \label{eq:as-B}
\end{align}
To see this suppose that
\begin{align*}
    \bbP_\mu(\limsup_{t\to\infty} |\mu_t| < \infty) > 0\,. 
\end{align*}
Then there exists a $ B > 0$ such that
\begin{align*}
    \bbP_\mu(\text{exists } t : \mu_t \notin [-B, B]) < 1\,,
\end{align*}
which is a contradiction.
In order to prove (\ref{eq:as-B}) notice that compactness of $[0,B]$ and $[-B,0]$ and
the continuity of $F$ at $(\mu,0 )$ for all $\mu$ ensures there exists
an $\epsilon > 0$ depending only on $B$ such that
$1 - F(\mu,\epsilon) \geq \epsilon$ for all $\mu \in [0,B]$ and $F(\mu,-\epsilon) \geq  \epsilon$ for all $\mu \in [-B, 0]$.
Hence for $n = 1 + B/\epsilon$
\begin{align*}
    \inf_{\mu \in [-B, B]} \bbP_\mu(\text{exists } t \leq n : \mu_t \notin [-B,B]) \geq \epsilon^n\,.
\end{align*}
Let $E_k$ be the event that $\mu_t \in [-B,B]$ for all $t \in \{nk+1,\ldots,n(k+1)+1\}$.
Then
\begin{align*}
\bbP_\mu(\mu_t \in [-B, B] &\text{ for all } t) 
= \bbP_\mu(E_k \text{ for all } k \geq 0) \\
&= \prod_{k=0}^\infty \bbP_\mu(E_k \mid E_1,\ldots,E_{k-1})\\
&\leq \prod_{k=0}^\infty (1 - \epsilon^n) = 0\,,
\end{align*}
which completes the proof.
\end{proof}

\begin{proof}[Proof of Theorem~\ref{thm:strong}]
We prove that if $B$ is sufficiently large and $\mu > B$ and $\tau = \min\{t : \mu_t \in [-B, B]\}$, then
\begin{align}
\bbP_\mu\left(\lim_{t\to\infty} \mu_t = \infty\right) > \frac{1}{2}
\label{eq:tech1}
\end{align} 
and 
\begin{align}
\bbP_\mu\left(\lim_{t\to\infty} \mu_t = \infty\text{ or } \tau < \infty\right) = 1\,.
\label{eq:tech2}
\end{align}
For $\mu < -B$ the same holds, but with $\mu_t$ tending to $-\infty$. To see why this implies the result, notice that the previous theorem shows that $\mu_t$ eventually leaves $[-B, B]$ almost surely. Each time this happens there is more than 0.5 probability of divergence and a certainty of either divergence or returning to $[-B, B]$. The conditional Borel-Cantelli theorem concludes the proof. To see
why (\ref{eq:tech1}) and (\ref{eq:tech2}) hold, let
\begin{align*}
M_n = \sum_{t=1}^n (\mu_{t+1} - \mu_t - \bar f(\mu_t))\,,
\end{align*}
which is a martingale with bounded increments since $|\mu_{t+1}-\mu_t|$ are bounded. Now $M_{\tau \wedge n}$ is also a martingale. Then by the strong law of large numbers for martingales,
\begin{align*}
    \lim_{n\to\infty} M_{\tau \wedge n} / n = 0 \text{ a.s.}\,,
\end{align*}
which implies that either $\tau < \infty$ or $\mu_n \to \infty$ almost surely. To see the latter case, suppose $\mu$ is sufficiently large, we compute
\begin{align*}
    \mu_{n+1} &= M_n + \mu_1 + \sum_{t=1}^n f(\mu_t)\\ 
              &> M_n + \mu_1 + \epsilon n \rightarrow \infty
\end{align*}
since $M_n/n\rightarrow 0$ a.s.
A similar argument applies when $\mu$ is sufficiently small. Finally, Eq. (\ref{eq:tech1}) holds by Azuma's inequality.
\end{proof}

\section{More on User Interest Dynamics} \label{app:deterministic}
\subsection{Linear Deterministic Model}
Recall that since $a$ is fixed, to simplify the notation, we write $\mu_t, c_t$ instead of $\mu_t(a), c_t(a)$ in this section. We assume $c_t = g(\mu_t)$ for some deterministic function $g$ 
and analyze the one-dimensional, autonomous, first-order discrete dynamical system:
\begin{align}
\mu_{t+1} = \mu_t + h(c_t) =  \mu_t + f(\mu_t), \label{eq:u_ds}
\end{align}
where $f = h\circ g$.
In the simple case in which $f$ is a linear function, i.e. $f(\mu_t)=k\mu_t + b$ for a linear coefficient $k$ and constant term $b$, Eq.~\ref{eq:u_ds} takes the form  
\begin{align}
\mu_{t+1} &= (1+k) \mu_t + b.\label{eq:linear_u}
\end{align}
By unrolling \eqref{eq:linear_u} over time we obtain
\begin{align}
    \mu_t &= (1+k)^t \mu_0 + b\sum_{i=0}^{t-1} (1+k)^i\\
    &=
    \begin{cases}
    \mu_0+bt, \text{ if } k=0\\
    (\mu_0 + \frac{b}{k})(1+k)^t - \frac{b}{k}, \text{ if } k\neq 0.\label{eq:u_linear_sol}
    \end{cases}
\end{align}
For $k=0$, the user's interest is not influenced by the recommender system but drifts with the constant $b$ -- this case occurs in the real world with probability 0.\\  
For $k\neq 0$, the steady state equilibrium, obtained from solving the equation $\bar{\mu} = \bar{\mu} + k\bar{\mu} + b$, is given by $\bar{\mu} = -\frac{b}{k}$. Substituting $-\frac{b}{k}$ with $\bar{\mu}$ in Eq.~\ref{eq:u_linear_sol} and taking the limit ${t\rightarrow\infty}$ on both sides, we obtain
\begin{align}
    \lim_{t\rightarrow\infty}\mu_t = (\mu_0 -\bar{\mu})\lim_{t\rightarrow\infty}(1+k)^t +\bar{\mu},
\end{align}
namely,
\begin{align}
    &\lim_{t\rightarrow\infty}\mu_t =
    \begin{cases}
    \bar{\mu}, \text{ if } |1+k|<1,\\
    \bar{\mu}, \text{ if } |1+k|>1 \text{ and } \mu_0 = \bar{\mu},\\
    \text{alternating } \{\mu_0, 2\bar{\mu}-\mu_0\},  \text{ if } k=-2,\\
    \infty \text{ if } 1+k>1, \label{eq:u_linear_sol2}
    \end{cases}\\
    &\text{ and, }
    \limsup_{t\rightarrow\infty}\mu_t = \infty \text{ when } 1+k<-1.
\end{align}
Therefore, no matter how the system selects items, in the first three cases of \eqref{eq:u_linear_sol2} the user interest model over the item $a$ in question is always bounded over time. Of these, only the first case $|1+k|<1$, or equivalently $-2<k<0$, will occur with probability different from 0.

On the other hand, for $k>0$ the recommender system will degenerate strongly with $\mu_t$ growing at an exponential rate. For $k<-2$, the recommender system will degenerate weakly with $\sup_{i<t} \mu_i$ growing exponentially.

In summary, we can draw the following conclusions when an item $a$ is served \io: 1) for $k>0$ or $k<-2$, the user interest $\mu_t$ degenerates by growing exponentially, and therefore the recommender system needs to exert control on how frequent such items are shown to the user in order to control the speed of degeneracy of $\mu_t$;
2) in all other cases, the user interest $\mu_t$ does not degenerate ($-2\leq k<0$), or the system cannot control linear degeneracy ($k=0$, improbable case).

\subsection{Non-linear Deterministic Model}
If $f$ is a non-linear deterministic function, the steady state equilibrium is reached at zeros of $f$. Sufficient (but not necessary) conditions for global stability are given by the following theorem~\citep{galor07discrete}:
\begin{theorem}[sufficiency]
If $g:\mathbb{R}\rightarrow\mathbb{R}$ is a contraction mapping, i.e. if 
\[
\frac{|g(y_{t+1})-g(y_t)|}{|y_{t+1}-y_t|} < 1, \hskip0.1cm \forall t=0,1,2,\ldots, \infty,
\]
then a stationary equilibrium of the difference equation $y_{t+1}=g(y_t)$ exists and is unique and globally (asymptotically) stable.
\label{thm:nonlinear_u}
\end{theorem}

Thus if $g=I+f$ (where $I$ is the identity function) is a contraction mapping, then there exists a steady state equilibrium that is unique and globally stable, and the system will not degenerate. Both requiring a globally stable steady state equilibrium and requiring $g$ to be a contraction mapping are strong sufficient conditions. Moreover it is almost always impossible to verify it in practice since we don't know the actual function $g$. Thus we gave three examples of sufficient conditions that are used to describe the general dynamics of user interests. In the first example, users respond to an item if their interests level exceed some action threshold. With each feedback, their interest level changes and so does the action threshold. If after a while the changes in action thresholds will always be smaller than user interest changes and the user interest's total variation ranges over $\mathbb{R}$, then the system degenerates.
\begin{theorem}[sufficiency]
Let $\{d_1, d_2, \ldots\}$ be an infinite sequence, $d_t\in\mathbb{R}$. Then if the user interest for an item has the following dynamics,
\begin{align}
&\mu_t > d_t \Longleftrightarrow \mu_{t+1} > \mu_t, \label{eq:mu_thres}\\
&\exists t_0 > 0 \text{ s.t. } |d_{t+1}-d_t|\leq |\mu_{t+1}-\mu_t|, \forall t\geq t_0, \label{eq:magnitude_thres}\\
&\sum_{t=0}^\infty |\mu_{t+1} - \mu_t| = \infty, \label{eq:inf_thres}
\end{align}
then it degenerates strongly as $t \rightarrow \infty$.
\label{thm:nonlinear_u_threshold}
\end{theorem}

\begin{proof}
We prove that $\mu_t\rightarrow \pm\infty$ as $t\rightarrow \infty$. At time $t_0$, either $\mu_{t_0} > d_{t_0}$ or $\mu_{t_0}\leq d_{t_0}$. First we consider the case where $\mu_{t_0} > d_{t_0}$. At the next time step $t_0+1$, there are again two different cases:
\begin{enumerate}
    \item $d_{t_0+1} < d_{t_0}$: it implies that $\mu_{t_0+1} > \mu_{t_0} > d_{t_0} > d_{t_0+1}$ by Condition~\ref{eq:mu_thres}. 
    \item $d_{t_0+1} \geq d_{t_0}$: by Condition~\ref{eq:magnitude_thres},  $d_{t_0+1} - d_{t_0}\leq \mu_{t_0+1} - \mu_{t_0}$ and thus 
    \[
    \mu_{t_0+1} \geq d_{t_0+1} - d_{t_0} + \mu_{t_0} > d_{t_0+1} 
    \]
\end{enumerate}
using $\mu_{t_0} - d_{t_0}>0$. Hence in both cases we have $\mu_{t_0+1} > d_{t_0+1}$. Applying the same argument to every time step, we also have 
\begin{align*}
\mu_{t_0} > d_{t_0} &\Longrightarrow  \mu_{t_0+1} > d_{t_0+1}\\
&\Longrightarrow  \mu_{t_0+2} > d_{t_0+2}\\
&\Longrightarrow \ldots
\end{align*}
which implies
\[
\mu_{t_0} < \mu_{t_0+1} < \mu_{t_0+2} < \ldots
\]
By Condition~\ref{eq:inf_thres}, we conclude that if $\mu_{t_0} > d_{t_0}$ then $\mu_t\rightarrow \infty$ as $t\rightarrow \infty$.

In the other case, $\mu_{t_0}\leq d_{t_0} \Longleftrightarrow \mu_{t_0+1} \leq \mu_{t_0}$. Similarly at time $t_0+1$, we have either $d_{t_0+1} < d_{t_0}$ or $d_{t_0+1} \geq d_{t_0}$. Following the same argument as above and reversing the inequality signs, we have 
\begin{align*}
\mu_{t_0} \leq d_{t_0} &\Longrightarrow  \mu_{t_0+1} \leq d_{t_0+1}\\
&\Longrightarrow  \mu_{t_0+2} \leq d_{t_0+2}\\
&\Longrightarrow \ldots
\end{align*}
which implies
\[
\mu_{t_0} \leq \mu_{t_0+1} \leq \mu_{t_0+2} \leq \ldots
\]
By Condition~\ref{eq:inf_thres}, we conclude that if $\mu_{t_0} \leq d_{t_0}$ then $\mu_t\rightarrow -\infty$ as $t\rightarrow \infty$.
\end{proof}

\subsection{Scale-invariance}
\label{app:rescale}
Let $\M$ be the candidate pool of items, which is also the domain of function $\mu_t$.

\begin{remark}[scale-invariance]
Let $\mu_t:\M \rightarrow D$ for any open interval $D\subset\mathbb{R}$. All sufficiency theorems still hold if we define the system degeneracy for $\mu_t$ as the following. We call $\mu_t:\M\rightarrow D$ degenerate iff there exists a monotonic, continuous function $\psi:D\rightarrow\mathbb{R}$ such that $\limsup_{t\rightarrow \infty}\|\psi\circ\mu_t - \psi\circ\mu_0\|_2 = \infty$. 
\label{thm:rescale}
\end{remark}
If we define $\nu_t = \psi\circ\mu_t$, then all sufficiency theorems readily apply to $\nu_t$ and the conclusions hold for $\mu_t$ since $\|\nu_t - \nu_0\|_2 = \|\psi\circ\mu_t - \psi\circ\mu_0\|_2$.

\section{Degeneracy Speed Analysis}\label{app:linear_speed}

We analyze the system degeneracy $\|\mu_t-\mu_0\|_2$ (when the candidate pool $\M$ is finite) and $\sup_{a\in\M}|\mu_t(a)-\mu_0(a)|$ (when $\M$ is infinite) asymptotically in the order of $t$. As $t\rightarrow\infty$, $\mu_t\rightarrow\infty$ for the selected items. Thus the selected items are asymptotically almost surely clicked, and hence for any such item $a$, $|\mu_t(a)-\mu_0(a)| \approx \delta(a) \cdot T_a(t)$.

\paragraph{Finite Candidate Pool.}
\begin{align}
    \|\mu_t-\mu_0\|_2 &\approx \sqrt{\sum_a \left[\delta(a) \cdot T_a(t)\right]^2}.
\end{align} 
In the {\it Random} recommender system, $T_a(t) \approx tl/m$ and thus
\begin{align}
\|\mu_t-\mu_0\|_2 &\approx \sqrt{\sum_a \delta(a)^2} (l/m)\cdot t.
\end{align}
Both the {\it Oracle} and the {\it Optimal Oracle} have a fixed set of items $\s$ that they keep selecting. Thus $T_a(t) \approx t$ for any item $a\in\s$ and 
\begin{align}
\|\mu_t-\mu_0\|_2 &\approx \sqrt{\sum_{a\in\s} \delta(a)^2} \cdot t.
\end{align}
For both the {\it UCB} and {\it TS} we have 
\begin{align}
\|\mu_t-\mu_0\|_2 &\approx (\sqrt{\sum_a \delta(a)^2} / \sqrt{m^*}) \cdot t,
\end{align}
where $m^*$ are the number of close to optimal arms. Therefore all five models have the degeneracy quantity $\|\mu_t-\mu_0\|_2$ converge to a linear function of $t$. 

\paragraph{Infinite Candidate Pool.}
Since we sample $\delta(a)$ from a bounded interval $[-b, b]$, asymptotically $\sup_{a\in\M} \delta(a) \approx b$
\[
\sup_{a\in\M}|\mu_t(a)-\mu_0(a)| \approx \sup_{a\in\M} \delta(a) \cdot T_a(t) \approx b  T_{a^*}(t)
\]
where $a^*$ is an item with $\delta(a^*) \approx b$.

In the {\it Random} recommender system, $T_a(t) \approx const$ and thus
\begin{align}
\sup_{a\in\M}|\mu_t(a)-\mu_0(a)| &\approx const.
\end{align}
The {\it Oracle} has a fixed set of items $\s$ that it keeps selecting. Thus $T_a(t) \approx t$ for any item $a\in\s$ and 
\begin{align}
\sup_{a\in\M}|\mu_t(a)-\mu_0(a)| &\approx \max_{a\in\s} \delta(a) \cdot t.
\end{align}
The {\it Optimal Oracle} model will pick all items with $\delta(a) \approx b$ and $T_a(t) \approx ct$ for some constant $c$.
\begin{align}
\sup_{a\in\M}|\mu_t(a)-\mu_0(a)| &\approx b c \cdot t.
\end{align}
For UCB and Thompson sampling the situation is less clear. When growth of the candidate pool is linear, then UCB will spend most of its time exploring new items and consequently degeneration is very slow. Thompson sampling will continue to play degenerate items with reasonable probability and so degeneration speed will be larger than for UCB. Precisely quantifying the rates of degeneration depends in a complicated way on the details of the model.

\end{appendices}

\end{document}